\documentclass[letterpaper, 10 pt, conference]{ieeeconf}  

\IEEEoverridecommandlockouts      

\usepackage{comment}
\usepackage{graphicx}
\usepackage{multirow}
\usepackage{amssymb}
\usepackage{algorithm}
\usepackage{algorithmic}

\usepackage{amsmath}
\usepackage{booktabs} 
\usepackage{xcolor}
\usepackage{caption}
\usepackage{subcaption}

\newtheorem{definition}{Definition}

\newtheorem{remark}{Remark}
\newtheorem{corollary}{Corollary}
\newtheorem{lemma}{Lemma}

\begin{document}


\title{\LARGE \bf
  Gaussian Mixture-Based Inverse Perception Contract for Uncertainty-Aware Robot Navigation
}
\author{Bingyao Du$^1$, Joonkyung Kim$^2$ and Yiwei Lyu$^2$
\thanks{$^1$ The author is with the Department of Computer Science, Columbia University, New York, NY 10027, USA. Email: \tt \small  bd2747@columbia.edu}
\thanks{$^{2}${The authors are with the Department of Computer Science and Engineering, Texas A\&M University, College Station, TX 77843, USA. Email: \tt \small  joonkyung,yiweilyu@tamu.edu.}}
}

\maketitle
\thispagestyle{empty}
\pagestyle{empty}

\begin{abstract}                
Reliable navigation in cluttered environments requires perception outputs that are not only accurate but also equipped with uncertainty sets suitable for safe control. An inverse perception contract (IPC) provides such a connection by mapping perceptual estimates to sets that contain the ground truth with high confidence. Existing IPC formulations, however, instantiate uncertainty as a single ellipsoidal set and rely on deterministic trust scores to guide robot motion. Such a representation cannot capture the multi-modal and irregular structure of fine-grained perception errors, often resulting in over-conservative sets and degraded navigation performance. In this work, we introduce Gaussian Mixture-based Inverse Perception Contract (GM-IPC), which extends IPC to represent uncertainty with unions of ellipsoidal confidence sets derived from Gaussian mixture models. This design moves beyond deterministic single-set abstractions, enabling fine-grained, multi-modal, and non-convex error structures to be captured with formal guarantees. A learning framework is presented that trains GM-IPC to account for probabilistic inclusion, distribution matching, and empty-space penalties, ensuring both validity and compactness of the predicted sets. We further show that the resulting uncertainty characterizations can be leveraged in downstream planning frameworks for real-time safe navigation, enabling less conservative and more adaptive robot motion while preserving safety in a probabilistic manner.

\end{abstract}



\section{Introduction}
Autonomous systems are increasingly deployed in safety-critical domains such as urban driving, aerial delivery, and indoor service robotics. In these settings, perception errors, such as missed detections, inaccurate localization, or delayed updates, can quickly propagate into unsafe control actions~\cite{dean2020robust}. Robust control mitigates such risks by assuming bounded disturbances but often sacrifices efficiency for safety. Achieving safe yet effective autonomy thus requires perception outputs that are accurate and accompanied by uncertainty sets suitable for decision making and control~\cite{shao2024uncertainty}.

Existing methods quantify uncertainty in different ways. Robust control theory provides worst-case guarantees but is overly conservative in complex environments~\cite{berberich2025overview}. Bayesian deep learning~\cite{guo2017calibration} produces probabilistic outputs without certified guarantees once embedded in closed-loop systems, while conformal prediction~\cite{fontana2023conformal} offers distribution-free confidence sets but is limited to low-dimensional outputs. Perception contract (PC~\cite{hsieh2022verifying}) formalizes bounded perception errors for control synthesis, and its inverse form (IPC~\cite{sun2024learning}) directly maps perception estimates to sets guaranteed to contain the ground truth. IPCs are attractive because they can wrap around existing perception modules to generate uncertainty sets for safe navigation.

However, existing IPC formulations instantiate uncertainty as a single ellipsoidal set and rely on deterministic trust scores to guide robot motion.
This abstraction is suitable for simple outputs but fails for complex perception tasks like segmentation or point-cloud estimation, since it cannot capture the multimodal and irregular structure of fine-grained perception errors, which often results in over-conservative sets and degraded navigation performance.

In this work, we introduce Gaussian Mixture-Based Inverse Perception Contract (GM-IPC), which extends IPC to represent uncertainty with unions of ellipsoidal confidence sets derived from Gaussian mixture models. This design moves beyond deterministic single-set abstractions, enabling fine-grained, multimodal, and non-convex error structures to be captured with formal guarantees. A learning framework is presented that trains GM-IPC to account for probabilistic inclusion, distribution matching, and empty-space penalties, ensuring both validity and compactness of the predicted sets. Our main contributions are:
1) We formalize the representation of GM-IPC as a union of ellipsoidal confidence regions, allowing uncertainty to be expressed in multimodal and irregular forms beyond the limitations of single-ellipsoid IPC; 2) We propose a learning framework for training GM-IPC from perception data, together with theoretical guarantees on probabilistic coverage under mild assumptions; and 3) We demonstrate how the information encoded in GM-IPC can be exploited to inform downstream planning, enabling adaptive collision-avoidance maneuvers around obstacles of varying predicted significance during navigation.

\section{Related Work}
\subsection{Perception-based Safe Navigation}

Perception-based navigation has become a dominant paradigm in both autonomous driving and mobile robotics. Typical pipelines combine raw sensor inputs with learned or model-based perception modules and planners, often without explicit environment modeling. In autonomous driving, common perception formats include 2D images, semantic segmentation maps, or bird’s-eye-view feature maps extracted from multi-camera systems and LiDAR sensors (e.g., \cite{Caesar2020nuScenes, Lang2019PointPillars,Philion2020LiftSplatShoot,Janai2020AutonomousVisionSurvey}). For indoor mobile robots, point cloud–based representations are widely used, often paired with learned 3D detection or mapping modules (e.g., \cite{wang2019pseudo, votenet, liu2021group, yang2023sam3d}). VoteNet is a representative example that directly processes 3D point clouds to detect objects in cluttered indoor scenes. While such pipelines achieve strong empirical performance, they generally lack formal safety guarantees. The reliability of the controller depends on the training distribution and may degrade under domain shifts, sensor noise, or occlusions~\cite{lindemann2021robust}. This limitation poses significant risks in safety-critical settings such as healthcare, urban driving, or human–robot interaction.

Uncertainty in perception-based navigation arises from multiple sources~\cite{hsieh2022verifying, resiliency}: sensor-level noise and occlusion create missing or corrupted data; model-level uncertainty emerges from limited training data, distribution shifts, and inductive biases of neural architectures. For example, convolution assumes local spatial correlation. In navigation, a CNN may wrongly assume nearby pixels represent nearby obstacles, causing confusion in occluded scenes; task-level uncertainty stems from partial observability and temporal aliasing, where the same sensor input may correspond to multiple plausible world states. Without explicitly modeling these uncertainties, systems remain vulnerable to unsafe decision-making.
\vspace{-.4em}
\subsection{Uncertainty Modeling and Quantification for Perception}
\vspace{-.7em}
A central challenge in perception-based control is quantifying uncertainty for perception modules. PC~\cite{hsieh2022verifying, astorga2023perception} and IPC~\cite{sun2024learning} address this by representing perception errors with structured uncertainty sets. PC certifies that the perception error lies within a known bound, while IPC maps perceptual estimates directly to uncertainty sets guaranteed to contain the ground truth. These frameworks enable certified uncertainty wrappers around existing perception modules and integrate naturally with safe control synthesis. For low-dimensional outputs such as object positions, simple convex sets such as ellipsoids are typically sufficient.


As perception modules advance, however, their outputs become higher-dimensional and less structured. Fine-grained perceptual outputs—such as point clouds, bounding boxes, or segmentation masks—exhibit uncertainty that is multimodal (e.g., two disjoint regions may both plausibly contain an obstacle), irregular in shape (e.g., uncertainty may follow the contour of furniture rather than a convex blob), and history-dependent (e.g., occlusions may be resolved only after multiple observations over time)~\cite{chen2025hyperdimensional, hyper_seg}. Single-ellipsoid modeling is either overly conservative or fails to capture true obstacle geometry, leading to unsafe maneuvers.

Recent work has built high-confidence uncertainty sets within perception training or robust planning~\cite{ren2024recursively}. In contrast, we propose an expressive uncertainty representation, Gaussian Mixture–based IPC, that wraps around any vision-based perception module and integrates it directly with safe planners. By modeling multimodal, irregular errors as unions of confidence ellipsoids, GM-IPC enables adaptive conservatism that avoids both overreaction and underestimation in navigation.

\section{Method}
\label{sec:method}
\subsection{Problem Statement and Notation}
We consider a robot navigating in an environment with obstacles detected by a perception module. Let $x \in \mathcal{X}$ denote the robot state, $\hat{y} \in \mathcal{Y}$ the perceptual estimate, and $y \in \mathcal{Y}$ the unknown ground truth. In our experiment, $\hat{y}$ represents a subset of all points captured by the sensor, consisting of those classified by the perception algorithm as obstacle points. The goal of an IPC is to construct a set
\begin{equation}
A_\theta(x,\hat{y}) \subseteq \mathcal{Y},
\end{equation}
parameterized by $\theta$, such that with high probability
\[
y \in A_\theta(x,\hat{y}).
\]

This set acts as a safety certificate: if the planner synthesizes a control input that avoids $A_\theta(x,\hat{y})$, then the ground-truth obstacle will also be avoided with the same probability.

Classical IPCs construct $A_\theta$ as a single deterministic ellipsoid or sphere, which is too coarse to capture fine-grained perceptual uncertainty. Our goal is to design an IPC representation and training procedure that (i) better models complex error distributions, (ii) admits probabilistic coverage guarantees, and (iii) integrates naturally with downstream safe control.

\subsection{Representing Perceptual Uncertainty with Gaussian Mixtures}

\begin{figure*}[t]
    \centering
    \begin{subfigure}[t]{0.31\textwidth}
        \includegraphics[height=0.78\textwidth]{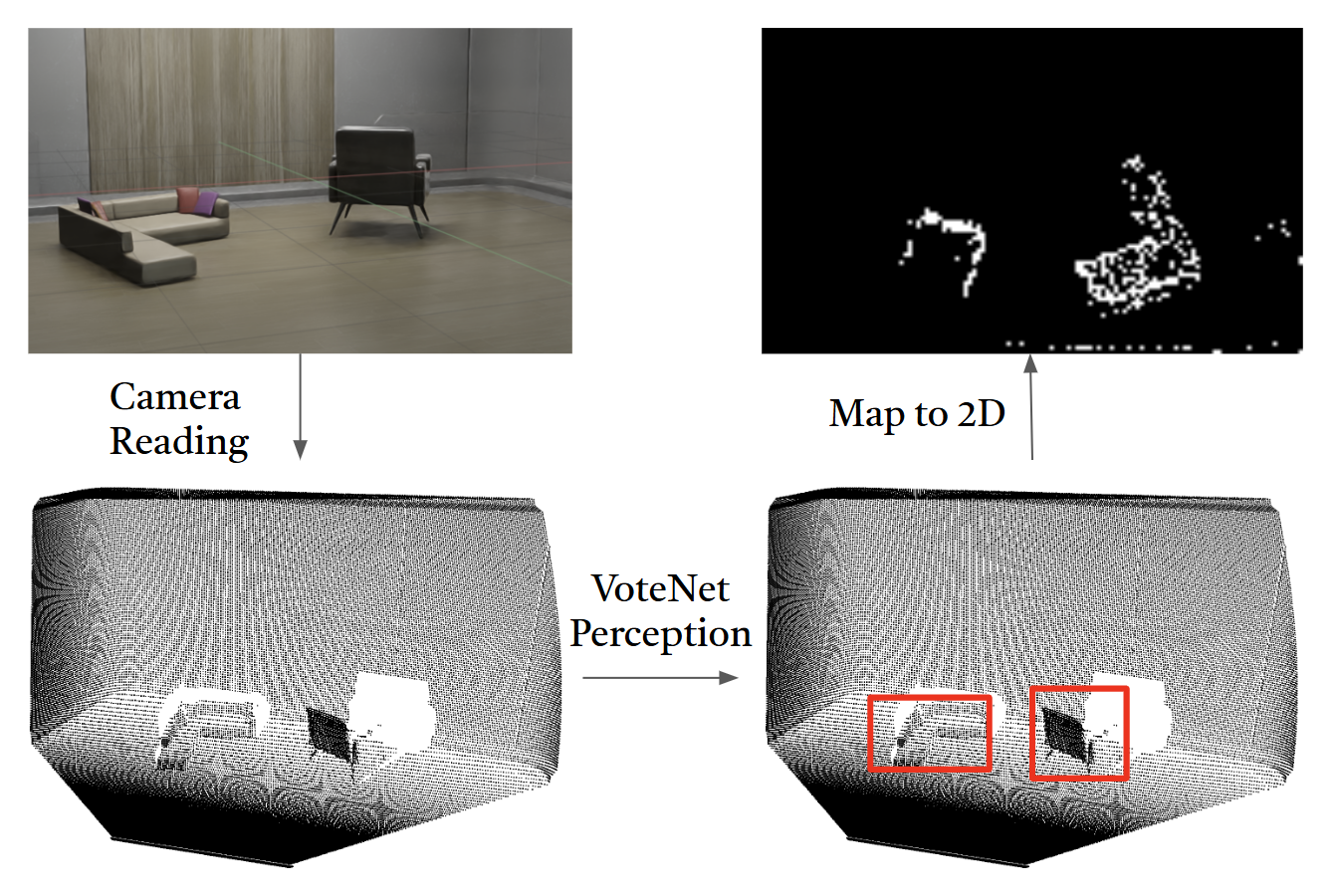}
        \caption{}
    \end{subfigure}
    \hspace{0.015\textwidth}
    \begin{subfigure}[t]{0.31\textwidth}
        \includegraphics[height=0.80\textwidth]{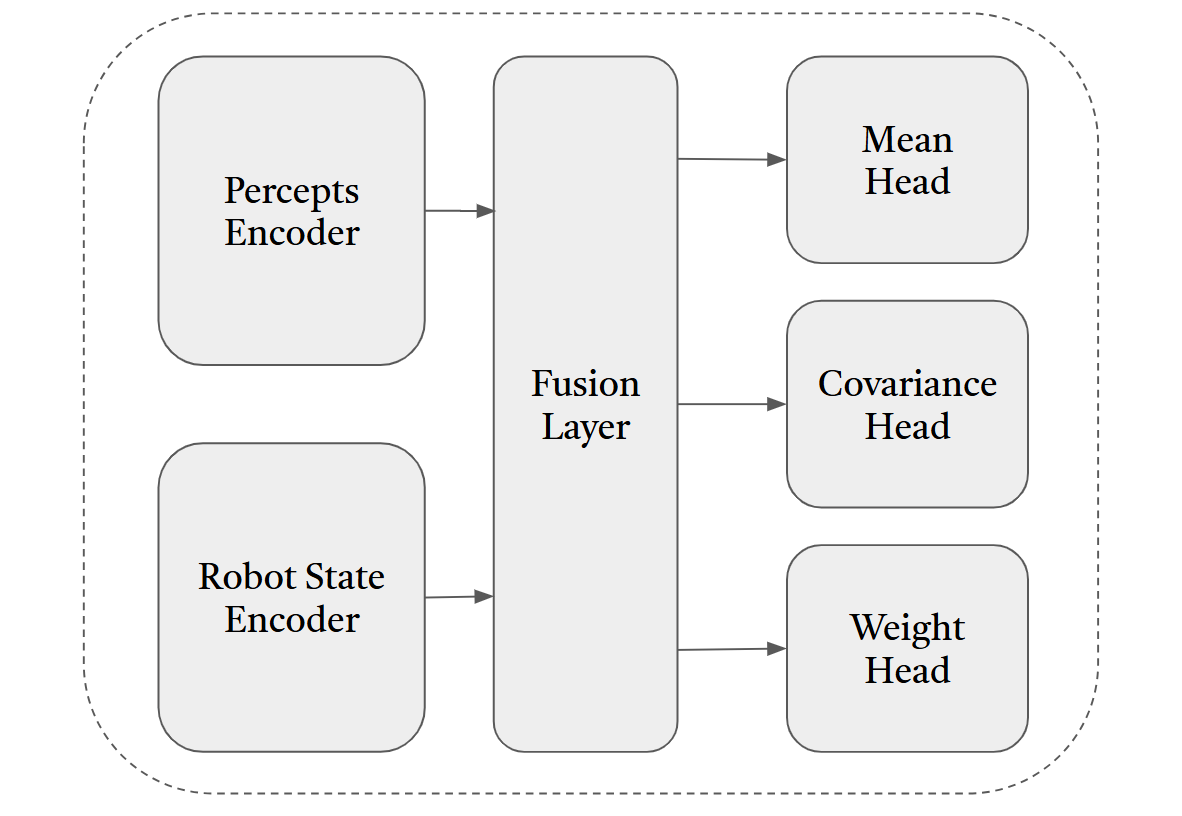}
        \caption{}
    \end{subfigure}
    \hspace{0.015\textwidth}
    \begin{subfigure}[t]{0.31\textwidth}
        \includegraphics[height=0.78\textwidth]{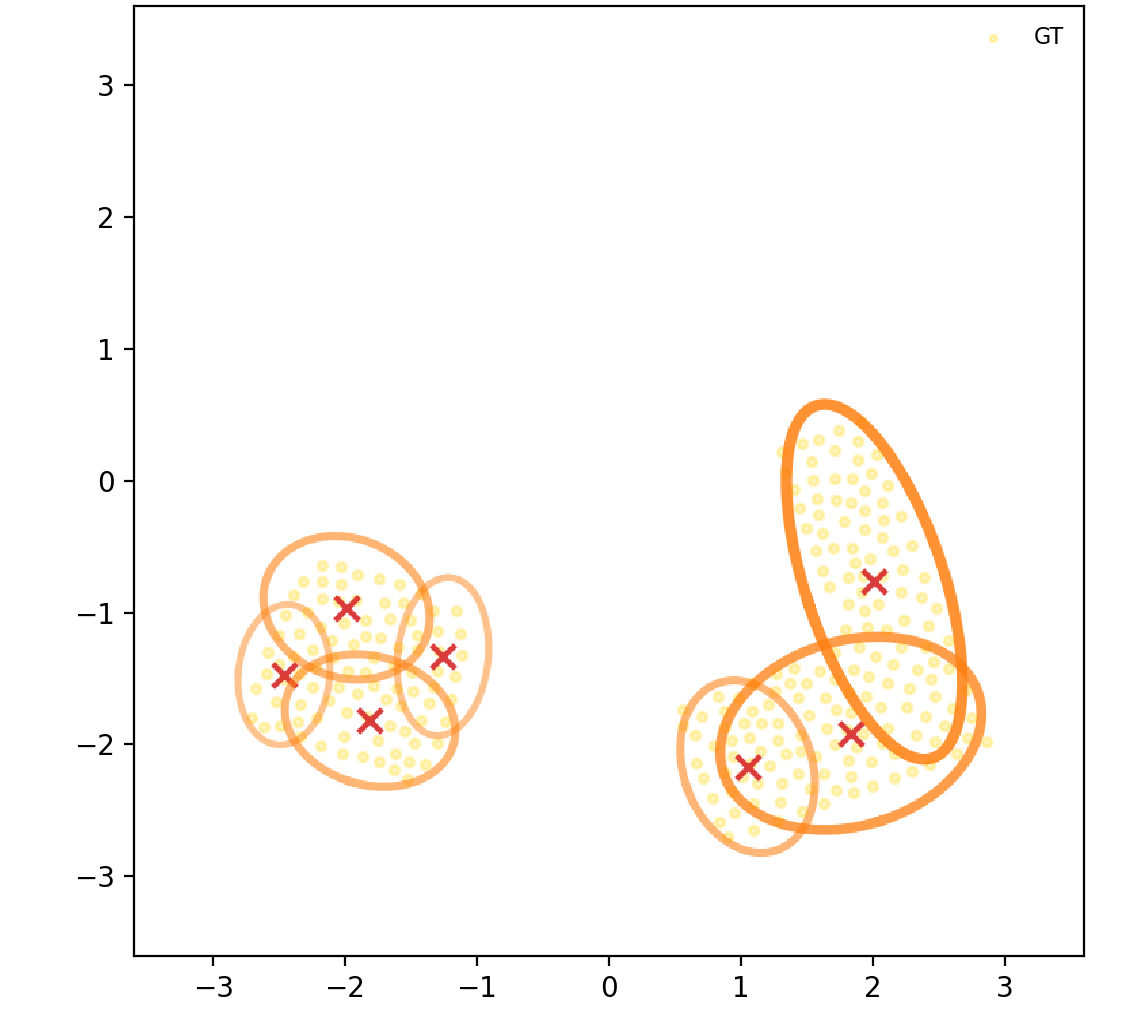}
        \caption{}
    \end{subfigure}
    \caption{Overview of GM-IPC. (a) Perception module: the robot camera captures the point cloud of environment, and the perception module predicts obstacle bounding boxes. Points inside the boxes are mapped to a 2D binary occupancy map. 
    (b) IPC network: The map is processed by a convolutional encoder, with the robot state processed by an MLP. Their representations are then fused through one fusion layer, followed by 3 linear heads to produce the Gaussian parameters
    (c) Uncertainty sets: given a confidence value $p$, the boundaries of all Gaussians are extracted to cover the obstacle region.}
    \label{fig:learning_gmm_ipc}
\end{figure*}

We model perception uncertainty using a Gaussian mixture model~\cite{em_algo, bishop2006pattern}. Each component $i \in \{1,\dots,K\}$ is a Gaussian distribution with mean $\mu_i$ and covariance $\Sigma_i \succ 0$. Here, $K$ is the number of Gaussian components used to represent the uncertainty set. In practice, we parameterize $\Sigma_i$ as 
$LL^T$-like form with softplus-positive diagonals, and add a small diagonal regularization term to ensure strict positive definiteness numerically. This induces ellipsoidal confidence regions, which we use to construct Inverse Perception Contracts.

\medskip
\begin{definition}[Ellipsoidal Confidence Region]
For $Y_i \sim \mathcal{N}(\mu_i, \Sigma_i)$ in $\mathbb{R}^d$, the ellipsoidal confidence region at level $\rho \in (0,1)$ is
\begin{equation}
E(\mu_i, \Sigma_i, \tau) = \{ y \in \mathbb{R}^d : (y-\mu_i)^\top \Sigma_i^{-1}(y-\mu_i) \leq \tau \},
\end{equation}
where $\tau = \chi^2_{d,\rho}$ is the chi-square quantile.
\end{definition}

By construction, since the quadratic form $(Y_i - \mu_i)^\top \Sigma_i^{-1}(Y_i - \mu_i)$ follows a chi-square distribution with $d$ degrees of freedom, the ellipsoid $E(\mu_i, \Sigma_i, \chi_{d, \rho}^2)$ contains the true value with probability $\rho$.


\medskip
\begin{definition}[Gaussian Mixture–based IPC]
Given state $x \in \mathcal{X}$ and perceptual estimate $\hat{y} \in \mathcal{Y}$, a GM-IPC is defined as
\begin{equation}
A_\theta(x, \hat{y}) = \bigcup_{i=1}^K E\!\left(\mu_i(x,\hat{y}), \Sigma_i(x,\hat{y}), \chi^2_{d,\rho}\right).
\end{equation}
The IPC condition is satisfied if the ground truth $y$ lies inside $A_\theta(x,\hat{y})$.
\end{definition}

\medskip
\begin{remark}
Classical IPCs correspond to the special case $K=1$. GM-IPCs generalize this formulation by using unions of ellipsoids, enabling multimodal and irregular uncertainty patterns to be captured while retaining a probabilistic coverage interpretation.
\end{remark}

\vspace{-0.2em}
\subsection{Learning GM-IPC} \vspace{-0.5em}

The parameters of $A_\theta(x,\hat{y})$ are learned from data so that the resulting sets cover ground-truth obstacles with high probability while remaining compact. We optimize a composite objective
\begin{equation}
L(\theta) \;=\; L_{\mathrm{incl}} \;+\; \alpha\,L_{\mathrm{nll}} \;+\; \beta\,L_{\mathrm{empty}},
\end{equation}
where $L_{\mathrm{incl}}$ enforces probabilistic coverage, $L_{\mathrm{nll}}$ encourages component specialization, and $L_{\mathrm{empty}}$ penalizes spurious coverage of free space. Scalars $\alpha,\beta \ge 0$ control the trade-off.

\textbf{Inclusion loss.}
Let the training sample at time $t$ provide robot state $x_t$, perceptual estimate $\hat{y}_t$, and ground-truth obstacle points $\mathcal{V}_t=\{y\}$. For each component $i$, define a \emph{soft membership function}
$$
q_i:\mathbb{R}^d \rightarrow (0,1), 
\qquad 
$$
\begin{equation}
q_i(y) \;=\; \sigma\!\left(\frac{\chi^2_{d,\rho} - (y-\mu_i)^\top \Sigma_i^{-1} (y-\mu_i)}{s}\right),
\end{equation}
where $\sigma(\cdot)$ is the logistic function and $s>0$ is a fixed smoothing scale. The weighted probability that $y$ is covered by at least one component is
\begin{equation}
\begin{aligned}
p_{\mathrm{cover}}(y) 
&= 1 - \exp\!\left(\sum_{i=1}^K w_i \,\log\big(1 - q_i(y)\big)\right) \\
&= 1 - \prod_{i=1}^K \big(1 - q_i(y)\big)^{w_i}.
\end{aligned}
\end{equation}
The inclusion loss is then
\begin{equation}
L_{\mathrm{incl}}(\theta; x_t,\hat{y}_t,\mathcal{V}_t)
\;=\; -\frac{1}{|\mathcal{V}_t|}\sum_{y\in\mathcal{V}_t}\log p_{\mathrm{cover}}(y).
\end{equation}
Using $-\log u \ge 1-u$ for $u\in(0,1]$, this directly upper-bounds the miscoverage probability:
\begin{equation}
\mathbb{E}\!\big[\mathbf{1}\{y \notin A_\theta(x_t,\hat{y}_t)\}\big]
\;\le\;
\mathbb{E}\!\big[-\log p_{\mathrm{cover}}(y)\big].
\label{eq:upperbound_miscoverage}
\end{equation}
For reference, the ideal non-smooth formulation uses hard set indicators
$\mathbf{1}\{y \in E_i\}$ with $E_i = E(\mu_i,\Sigma_i,\chi^2_{d,\rho})$:
\begin{equation}
\begin{aligned}
p_{\mathrm{cover}}^{\star}(y)
&= 1 - \prod_{i=1}^K \bigl( 1 - \mathbf{1}\{ y \in E_i \} \bigr), \\
L_{\mathrm{incl}}^{\star}(\theta; x_t,\hat{y}_t,\mathcal{V}_t)
&= -\frac{1}{|\mathcal{V}_t|}\sum_{y\in\mathcal{V}_t} \log p_{\mathrm{cover}}^{\star}(y).
\label{eq:incl_loss}
\end{aligned}
\end{equation}

Because hard membership and the resulting OR over components are non-differentiable, we optimize the smooth surrogate in (5)–(6), which remains differentiable while still controlling miscoverage via an upper bound.

\textbf{Distribution-matching loss.}
To promote non-redundant components that specialize to different regions, we fit ground-truth points under the Gaussian mixture:
\begin{equation}
L_{\mathrm{nll}}(\theta; \mathcal{V}_t)
= -\frac{1}{|\mathcal{V}_t|}\sum_{y\in\mathcal{V}_t}
\log\!\left(\sum_{i=1}^K w_i \,\mathcal{N}\!\big(y \,\big|\, \mu_i,\Sigma_i\big)\right).
\label{eq:nll_loss}
\end{equation}

\textbf{Empty-space penalty.}
To discourage ellipsoids from extending into unsupported regions, we penalize coverage of free-space cells. Let $\Omega$ be a workspace discretization and $\Omega_{\mathrm{neg}}\subset\Omega$ the cells known to be obstacle-free. For cell $c$, define per-component soft membership $p_i(c)$ analogously to $q_i(\cdot)$ (e.g., by evaluating at the cell center). The union probability over the mixture is
\begin{equation}
p_{\mathrm{union}}(c) \;=\; 1 - \prod_{i=1}^K \big(1 - p_i(c)\big),
\end{equation}
and the empty-space penalty is
\begin{equation}
L_{\mathrm{empty}}(\theta; \Omega_{\mathrm{neg}})
\;=\; \frac{1}{|\Omega_{\mathrm{neg}}|}\sum_{c\in\Omega_{\mathrm{neg}}} \,\phi\!\big(p_{\mathrm{union}}(c)\big),
\label{eq:empty_loss}
\end{equation}
where $\phi(\cdot)$ is a weighting that emphasizes boundary or hard negatives (e.g., focal-style weighting).

\medskip
In implementation, we parameterize GM-IPC with a lightweight network that fuses perception and state (see Fig.~\ref{fig:learning_gmm_ipc}). Filtered point clouds are rasterized into a fixed-resolution occupancy map and encoded via a shallow convolutional encoder; the robot state $x_t$ is embedded with a small MLP. Fused features feed three linear heads that output $\{\mu_i\}_{i=1}^K$, $\{\Sigma_i\}_{i=1}^K$ (parameterized to ensure $\Sigma_i \succ 0$), and $\{w_i\}_{i=1}^K$ (normalized to range $(0, 1]$).

\subsection{Theoretical Guarantees}
\label{sec:theory}

We state generalization guarantees at the \emph{trial} level and connect them to probabilistic coverage.

\noindent\textbf{Setup.}
Let a trial be $Z=(x_{1:T}, \hat{y}_{1:T}, \mathcal{V}_{1:T})$. Trials $\{Z_m\}_{m=1}^M$ are generated by randomizing the environment (e.g., obstacle placement and start/goal) and then running a fixed controller; thus, ${Z_m}$ are i.i.d. — that is, each trial $Z_m$ is generated independently under identical experimental conditions and follows the same underlying distribution, even though samples within a trial are temporally correlated. $R(\theta)$ denotes the expected trial-level risk under the data-generating distribution,
while $\hat{R}_M(\theta)$ is its empirical estimate computed as the average loss over $M$ i.i.d.\ trials. Define
\begin{equation}
L_{\mathrm{trial}}(\theta; Z) = \frac{1}{T}\sum_{t=1}^T L_{\mathrm{incl}}(\theta; x_t,\hat{y}_t,\mathcal{V}_t),
\end{equation}
\begin{equation}
R(\theta) = \mathbb{E}[\,L_{\mathrm{trial}}\,], \quad \hat{R}_M(\theta) = \frac{1}{M}\sum_{m=1}^M L_{\mathrm{trial}}(\theta; Z_m).
\end{equation}

\noindent\textbf{From empirical loss to coverage.}
Section III.C established inequality~\ref{eq:upperbound_miscoverage}, which upper-bounds the (stepwise) miscoverage by the inclusion loss.
Averaging over $t$ shows that controlling $L_{\mathrm{trial}}$ controls the \emph{trial-averaged} miscoverage rate.

\medskip
\begin{lemma}[PAC generalization bound over trials]
\label{lem:pac}
Let $\mathcal{F}=\{ Z \mapsto L_{\mathrm{trial}}(\theta; Z): \theta\in\Theta\}\subset[0,1]$
have pseudo-dimension $P<\infty$.\footnote{Clipping $L_{\mathrm{incl}}$ to $[0,1]$ is standard to ensure $L_{\mathrm{trial}}\in[0,1]$; constants adjust accordingly.}
For any $\delta\in(0,1)$, with probability at least $1-\delta$ over $M$ i.i.d.\ trials,
the original PAC generalization bound~\cite{bartlett2002rademacher,shalev2014understanding} gives
\begin{equation}
R(\theta) \;\le\; \hat{R}_M(\theta)
\;+\; \sqrt{\frac{\ln(1/\delta)}{2M}}.
\end{equation}
Moreover, according to PAC theory based on uniform convergence with empirical
Rademacher complexity,
\begin{equation}
R(\theta) \;\le\; \hat{R}_M(\theta) \;+\; 2\,\mathcal{R}_M(\mathcal{F})
\;+\; \sqrt{\frac{\ln(1/\delta)}{2M}},
\end{equation}
where $\mathcal{R}_M(\mathcal{F})$ is the empirical Rademacher complexity. Furthermore,
$\mathcal{R}_M(\mathcal{F}) \le C\sqrt{P/M}$ for a universal constant $C$, yielding
\begin{equation}
R(\theta) \;\le\; \hat{R}_M(\theta) \;+\;
O\!\Big(\sqrt{\tfrac{P+\ln(1/\delta)}{M}}\Big).
\end{equation}

\end{lemma}

Combining (5) with Lemma~\ref{lem:pac} implies that, with probability at least $1-\delta$, a small empirical inclusion loss leads to a small \emph{true} miscoverage rate, provided the hypothesis class has bounded complexity. 

\medskip
\begin{remark}
(1) Trials are treated as i.i.d.\ draws even though samples within a trial are temporally correlated; the learning guarantees apply at the trial level. 
(2) \emph{Lemma~\ref{lem:pac}} provides the generalization bound that controls the realized miscoverage under the data-generating process.
\end{remark}

\vspace{-1em}
 
\medskip
\begin{corollary}[Consistency of ERM~\cite{shalev2014understanding} for GM-IPC]
\label{cor:consistency}
Let $\hat{\theta}_M \in \arg\min_{\theta\in\Theta} \hat{R}_M(\theta)$ be any empirical risk minimizer for the trial loss $L_{\mathrm{trial}}$.
If the uniform convergence in Lemma~\ref{lem:pac} holds and the population risk $R(\theta)$ has a unique minimizer $\theta^\star$, then
\begin{equation}
R(\hat{\theta}_M) \;\xrightarrow[M\to\infty]{\;\;P\;\;} R(\theta^\star).
\end{equation}
If the model is well specified (there exists $\theta^\dagger$ achieving the target coverage), then $R(\theta^\star)=R(\theta^\dagger)$; otherwise, $R(\theta^\star)$ is the best-in-class risk over $\Theta$.
\end{corollary}

\begin{proof}
Lemma~\ref{lem:pac} implies $\sup_{\theta\in\Theta}|R(\theta)-\hat{R}_M(\theta)| \to 0$ in probability. The inequalities below follow from uniform convergence, implying that minimizing $\hat{R}_M$ yields a solution whose population risk is close to optimal. For any $\varepsilon>0$, with high probability and large $M$,
\begin{equation}
R(\hat{\theta}_M) \le \hat{R}_M(\hat{\theta}_M) + \varepsilon \le \hat{R}_M(\theta^\star) + \varepsilon \le R(\theta^\star) + 2\varepsilon,
\end{equation}
hence $R(\hat{\theta}_M)\to R(\theta^\star)$ in probability.
\end{proof}

\begin{remark}
Corollary~\ref{cor:consistency} establishes statistical consistency of ERM, but does not guarantee convergence of the nonconvex training algorithm; in practice, we rely on approximate ERM via stochastic optimization.
\end{remark}

\vspace{-0.8em}
\subsection{Navigation Pipeline} \label{sec:nav_pip} 

Safe navigation requires explicit constraints that guarantee the robot will not enter obstacle regions predicted by the perception stack. Control Barrier Functions (CBFs~\cite{ames2019cbf}) provide such constraints by enforcing forward invariance of a safety set along the system dynamics. In our pipeline, the GM-IPC supplies a collection of ellipsoidal obstacle sets, which we convert into CBFs and embed in a model predictive controller.

\begin{figure}[h!]
  \centering
  \includegraphics[width=0.95\linewidth]{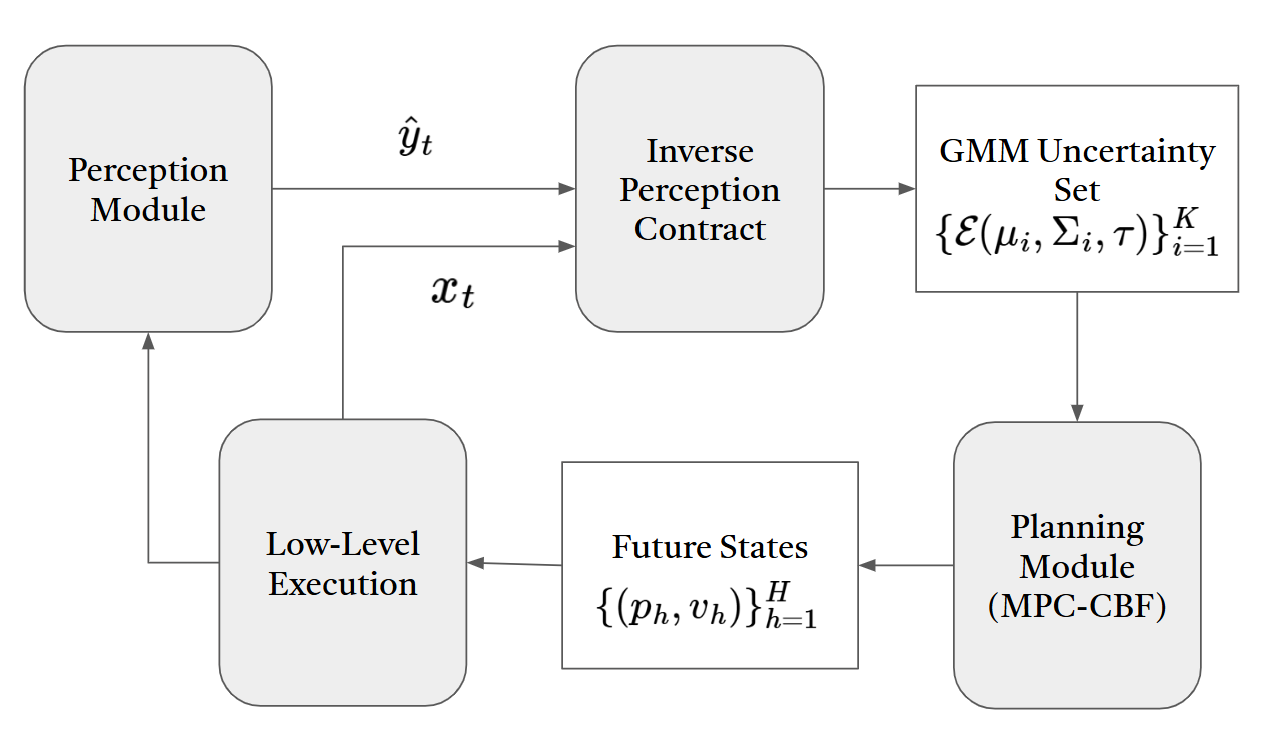}
  \caption{Overview of navigation pipeline. At each time step, raw sensor readings $\textnormal{P}_t$ are processed by the perception module to produce perceptual estimates $\hat{\textnormal{y}}_t$. The IPC network then generates a Gaussian Mixture uncertainty set, which is passed to the MPC-CBF planner to compute a sequence of safe future states ${(\textnormal{p}_h,\textnormal{v}_h)}_{h=1}^{H}$. The first control output is applied to the robot, and the process repeats at the next time step. White blocks denote intermediate results, while gray blocks represent modules.}  \label{fig:navigation_pipeline}
\end{figure}

\textbf{From Ellipsoids to Barrier Functions.}
Given a Gaussian component with mean $\mu_k \in \mathbb{R}^2$ and covariance $\Sigma_k \in \mathbb{R}^{2\times 2}$, we form a $\rho$ confidence ellipsoid (we choose $\rho =$ 95\% empirically) to represent the obstacle set seen by the controller. Let $\Sigma_k = V\Lambda V^\top$ with eigenvalues $\lambda_{\max}, \lambda_{\min}$ and principal eigenvector $[v_{11}, v_{21}]^\top$. With $\chi^2 = \chi^2_{2,0.95} = 5.991$, the semi-axes and orientation are
\begin{equation*}
a_k = \sqrt{\lambda_{\max} \chi^2}, \quad
b_k = \sqrt{\lambda_{\min} \chi^2}, \quad
\theta_k = \arctan2(v_{21}, v_{11}).
\end{equation*}
Let $(x_c,y_c)$ be the ellipsoid center. The associated barrier function
\begin{equation}
h_k(x) = 1 - \left(\tfrac{\tilde{x}}{a_k}\right)^2 - \left(\tfrac{\tilde{y}}{b_k}\right)^2, 
\; \;
\begin{bmatrix}\tilde{x}\\ \tilde{y}\end{bmatrix} = R(-\theta_k)
\begin{bmatrix}x-x_c \\ y-y_c\end{bmatrix},
\label{eq:hk_def}
\end{equation}
satisfies $h_k(x)\ge0$ outside the ellipsoid and $h_k(x)=0$ on its boundary. Enforcing a discrete-time CBF condition on $h_k$ yields forward invariance of the safe set, which is precisely the property needed for obstacle avoidance under time-varying perception.

\textbf{Perception-IPC Interface.}
At time $t$, the perception stack provides an occupancy representation $M_t$ and the robot state is $x_t$.
The IPC network maps $(M_t,x_t)$ to a Gaussian mixture
$\{(\mu_k,\Sigma_k,w_k)\}_{k=1}^K$, with $\mu_k\in\mathbb{R}^2$ and
$\Sigma_k\in\mathbb{R}^{2\times 2}$. Each $(\mu_k,\Sigma_k)$ induces a $95\%$
confidence ellipsoid via $\chi^2_{2,0.95}$ and the barrier function $h_k(x)$ in~\eqref{eq:hk_def}.
These barriers are updated online and used in the MPC-CBF(~\cite{mpccbf}) problem~\eqref{eq:nav_mpccbf}.

\textbf{Safe Planning via MPC-CBF.}
To leverage formal forward invariance properties with the probabilistic safety regions characterized by GM-IPC, we adopt the MPC-CBF formulation from~\cite{mpccbf} and incorporate the CBF constraints of all ellipsoids into the predictive optimization. At time~$t$, the controller solves \vspace{-0.5em}
\begin{equation}
\begin{aligned}
J_t^*(x_t) = \min_{u_{t:t+N-1|t}} \;\; & 
p(x_{t+N|t}) + \sum_{k=0}^{N-1} q(x_{t+k|t}, u_{t+k|t}) \\
\text{s.t.}\;\; 
& x_{t+k+1|t} = x_{t+k|t} + u_{t+k|t}\Delta t, \\
& x_{t+k|t} \in \mathcal{X}, \; u_{t+k|t} \in \mathcal{U}, \\
& x_{t|t} = x_t, \quad x_{t+N|t} \in \mathcal{X}_f, \\
& \Delta h(x_{t+k|t}, u_{t+k|t}) \geq -\gamma \, h(x_{t+k|t}),
\end{aligned}
\label{eq:nav_mpccbf}
\end{equation}
where $\Delta h(x,u) = h(x^+) - h(x)$ is the discrete-time CBF condition and $\gamma \in (0,1]$ is a relaxation parameter. Embedding~\eqref{eq:hk_def} into~\eqref{eq:nav_mpccbf} enforces forward invariance of the complement of each predicted ellipsoid, thereby guaranteeing obstacle avoidance consistent with the IPC.

To adapt conservativeness to perception quality, we use an adaptive relaxation $\gamma_{\text{ad}}$. First compute a confidence-weighted average
\begin{equation}
\bar{\rho} = 
\sum_{i \in \mathcal{O}}
\frac{ w^{\text{dist}}_i \, \rho_i }
     { \sum_{j \in \mathcal{O}} w^{\text{dist}}_j \, \rho_j }, 
\; \; \;
w^{\text{dist}}_i = \frac{1}{\|x - x_{c,i}\| + \epsilon},
\end{equation}
where $x\in\mathbb{R}^2$ is the robot position, $w^{\text{dist}}_i$ is an inverse-distance weight ($\epsilon>0$), and $\rho_i$ is the normalized confidence for obstacle $i$. The adaptive relaxation is
\begin{equation}
\gamma_{\text{ad}} = \gamma_{\min} 
+ (\gamma_{\max} - \gamma_{\min})
\cdot \frac{\bar{\rho} - \rho_{\min}}{\rho_{\max} - \rho_{\min}},
\end{equation}
with $\gamma_{\min}, \gamma_{\max} \in (0,1]$ user-defined and $\rho_{\min}, \rho_{\max}$ the minimum/maximum confidence levels. When $\bar{\rho}$ is small (nearby obstacles with low confidence dominate), $\gamma_{\text{ad}}\!\to\gamma_{\min}$, tightening the constraint $\Delta h \ge -\gamma_{\text{ad}} h$ and enlarging the effective safety margin. When obstacles are distant and $\rho_i$ is high, $\bar{\rho}$ increases and $\gamma_{\text{ad}}\!\to\gamma_{\max}$, relaxing the constraint and allowing more efficient motion.

\textbf{Execution and Iteration.}
The first control input from~\eqref{eq:nav_mpccbf} is applied, the robot transitions to the next state, and the cycle repeats with updated perception, IPC estimates, and CBFs.

\section{Experiments}

We evaluate GM-IPC on (i) \emph{validity and compactness} of learned uncertainty sets and (ii) \emph{utility for online navigation} with MPC-CBF controller. 

\subsection{Setup and Protocol}

\noindent\textbf{Simulation environment and robot:}
All experiments are run in Isaac Sim Simulator with the Nova Carter platform and a depth camera that provides point clouds in the camera frame.

\noindent\textbf{Data generation:}
Scenes are randomized for each trial, with variations in start/goal locations, obstacle configurations, and object instances. Point clouds are streamed at each step; the pipeline in Sec.~\ref{sec:method} produces ellipsoids used by the controller.

\noindent\textbf{Training protocol:}
Models are trained on randomized trials separate from evaluation scenes. Unless noted, the maximum number of ellipsoids is $K_{\max}{=}5$ (single-object) and $K_{\max}{=}7$ (multi-object); loss weights are fixed to $(\alpha,\beta)$ as in Sec.~\ref{sec:method} C. We use Adam, batch size 16, learning rate $1\mathrm{e}{-3}$ with cosine decay, and early stopping on validation inclusion loss.

\subsection{Scenarios}

We consider four indoor navigation scenarios of increasing geometric complexity, with a fixed arena size:
\begin{itemize}
    \item \textbf{Single chair (simple):} compact geometry, typically well-approximated by one ellipsoid.
    \item \textbf{Single L-shaped sofa (moderate):} non-convex footprint, challenging for single-ellipsoid coverage.
    \item \textbf{Multiple sofas (hard):} several instances of a complex shape.
    \item \textbf{Mixed objects (challenging):} one sofa + one chair (heterogeneous shapes).
\end{itemize}
Each scenario is evaluated over $20$ i.i.d.\ trials with randomized start/goal configurations and obstacle placement.

\subsection{Metrics}

\noindent \textbf{Inclusion rate:}
Fraction of ground-truth obstacle points inside the union of $95\%$ ellipsoids per step; we also report \emph{stepwise validity} (percentage of steps with $\ge 95\%$ inclusion).  

\noindent \textbf{Compactness:}
Inclusion normalized by set size: \# of ground-truth points inside the union divided by the union area. The union area is estimated by Monte Carlo method (\cite{kroese2013handbook}).

\noindent\textbf{Navigation performance:}
We report the number of control steps, path length, path efficiency (ratio between the Euclidean start–goal distance and the actual path length, notated as Eff. in Table.~\ref{tab:nav_eval}), average control time (mean computation time per control command, Ctrl.), and success rate (percentage of trials reaching the goal without collision, Succ.).

\noindent\textbf{Baselines:} Single-ellipsoid IPC with identical training and planner with $K{=}1$ and $K{=}2$.  

\noindent\textbf{Ablations:} Vary loss function by (i) removing $L_{\mathrm{nll}}$ or (ii) removing $L_{\mathrm{empty}}$.

\subsection{Validity of GM-IPC}

We evaluate validity in closed-loop rollouts across four scenarios. In Table~\ref{tab:validity}, we record \emph{stepwise validity} as the percentage of steps in which at least $95\%$ of ground-truth (GT) obstacle points are covered by at least one ellipsoid. The \emph{inclusion rate} is defined as the ratio of covered GT points to total GT points. \emph{Compactness} is the inclusion rate normalized by the union area (estimated via Monte Carlo sampling). We use Ellip-1 and Ellip-2 to denote the baseline ellipsoid-based IPC models with non-adaptive deterministic boundaries. For multi-obstacle case, Ellip-2 (two ellipsoids, one per obstacle) is included as an additional baseline.

\begin{table}[t]
\centering
\scriptsize
\setlength{\tabcolsep}{2pt}
\renewcommand{\arraystretch}{0.9}
\caption{Validity of IPC models.}
\begin{tabular}{llccccc}
\toprule
\multirow{2}{*}{Scenario} & 
\multirow{2}{*}{Model} &
\multicolumn{2}{c}{Mean Inclusion Rate} & 
\multicolumn{2}{c}{\% Steps $\geq 95\%$} &
\multicolumn{1}{c}{Compactness} \\
\cmidrule(lr){3-4} \cmidrule(lr){5-6} \cmidrule(lr){7-7}
 & & Final & All & Final & All & Final ($\mu \pm \sigma$) \\
\midrule
Chair      & GMM     & \textbf{1.000} & \textbf{0.999} & \textbf{100.0} & \textbf{99.8} & $9.52 \pm 0.46$ \\
           & Ellip   & 0.999 & 0.989 & \textbf{100.0} & 96.6 & \textbf{11.74 $\pm$ 0.35} \\
\midrule
Sofa       & GMM     & \textbf{0.997} & \textbf{0.993} & \textbf{100.0} & \textbf{97.8} & \textbf{10.50 $\pm$ 1.00} \\
           & Ellip   & 0.994 & 0.988 & 96.7           & 93.7           & $9.48 \pm 0.39$ \\
\midrule
Multi Sofa & GMM     & \textbf{0.999} & \textbf{0.998} & \textbf{100.0} & 99.0           & \textbf{5.58 $\pm$ 0.30} \\
           & Ellip-1 & 0.990 & 0.989 & \textbf{100.0} & 97.9           & $4.05 \pm 0.45$ \\
           & Ellip-2 & \textbf{0.999} & 0.997 & \textbf{100.0} & \textbf{99.4} & $4.60 \pm 0.22$ \\
\midrule
Mixed      & GMM     & \textbf{0.996} & \textbf{0.995} & \textbf{100.0} & \textbf{98.3} & \textbf{7.35 $\pm$ 0.48} \\
           & Ellip-1 & 0.983 & 0.984 & 90.0           & 90.9           & $4.87 \pm 0.68$ \\
           & Ellip-2 & 0.995 & 0.992 & 96.7           & 97.3           & $6.25 \pm 0.34$ \\
\bottomrule
\end{tabular}
\label{tab:validity}
\end{table}

As a proof of concept, Fig.~\ref{fig:navigation_pipeline} shows that GM-IPC adapts the number of active components: in the single-sofa case, two to three ellipsoids suffice for high compactness, while remaining components are automatically shrunk and assigned negligible weights, limiting their influence during navigation.

Table~\ref{tab:validity} indicates that GM-IPC achieves near-perfect final inclusion and high stepwise validity across all scenarios. In the Chair and Sofa cases, coverage is reliably maintained. In Multi-Sofa and Mixed scenes, inclusion remains very high while compactness decreases, reflecting the challenge of covering disjoint or heterogeneous shapes with a finite number of components; nevertheless, validity remains strong.

Overall, empirical miscoverage is rare: $\geq 97\%$ steps meet or exceed the $\geq 95\%$ requirement. The compactness trend aligns with intuition: simpler convex objects admit larger, more conservative ellipsoids, whereas multiple or irregular objects require tighter, more fragmented coverings.

\noindent\textbf{Ablation Study.} To demonstrate the effect of the regularization loss terms, we conducted an ablation study on the single-sofa case by removing (i) the NLL term (see Eq.~\ref{eq:nll_loss}) or (ii) the empty penalty term (see Eq.~\ref{eq:empty_loss}), while keeping all other training configurations identical and performing inference on the same test set.
As shown in Fig.~\ref{fig:ablation_loss}, when the empty-penalty term is removed, the IPC network tends to predict large, low-weight ellipsoids, which may cause the navigation to take inefficient detours. When the NLL term is removed, the predicted ellipsoids often overlap excessively and fail to fully cover the ground-truth region.

\begin{table}[h]
\centering
\scriptsize
\setlength{\tabcolsep}{3pt}
\renewcommand{\arraystretch}{0.8}
\caption{Ablation study on loss functions.}
\begin{tabular}{lccccc}
\toprule
\multirow{2}{*}{Loss Function} &
\multicolumn{2}{c}{Mean Inclusion Rate} & 
\multicolumn{2}{c}{\% Steps $\geq 95\%$} &
\multicolumn{1}{c}{Compactness} \\
\cmidrule(lr){2-3} \cmidrule(lr){4-5} \cmidrule(lr){6-6}
 & Final & All & Final & All & Final ($\mu \pm \sigma$) \\
\midrule
Full  & 0.997 & 0.993 & 100.0 & 97.8 & 10.50 $\pm$ 1.00 \\
No Empty Penalty &   1.0    &   1.0    &     100.0   &   100.0      & 1.00 $\pm$ 0.00 \\
No NLL           &   0.997    &   0.993   &    100.0    &   97.0   &  9.43 $\pm$ 0.64 \\
\bottomrule
\end{tabular}
\label{tab:ablation_loss}
\end{table}

\begin{figure}[h!]
    \centering
    \begin{subfigure}[b]{0.15\textwidth}
        \includegraphics[width=\textwidth]{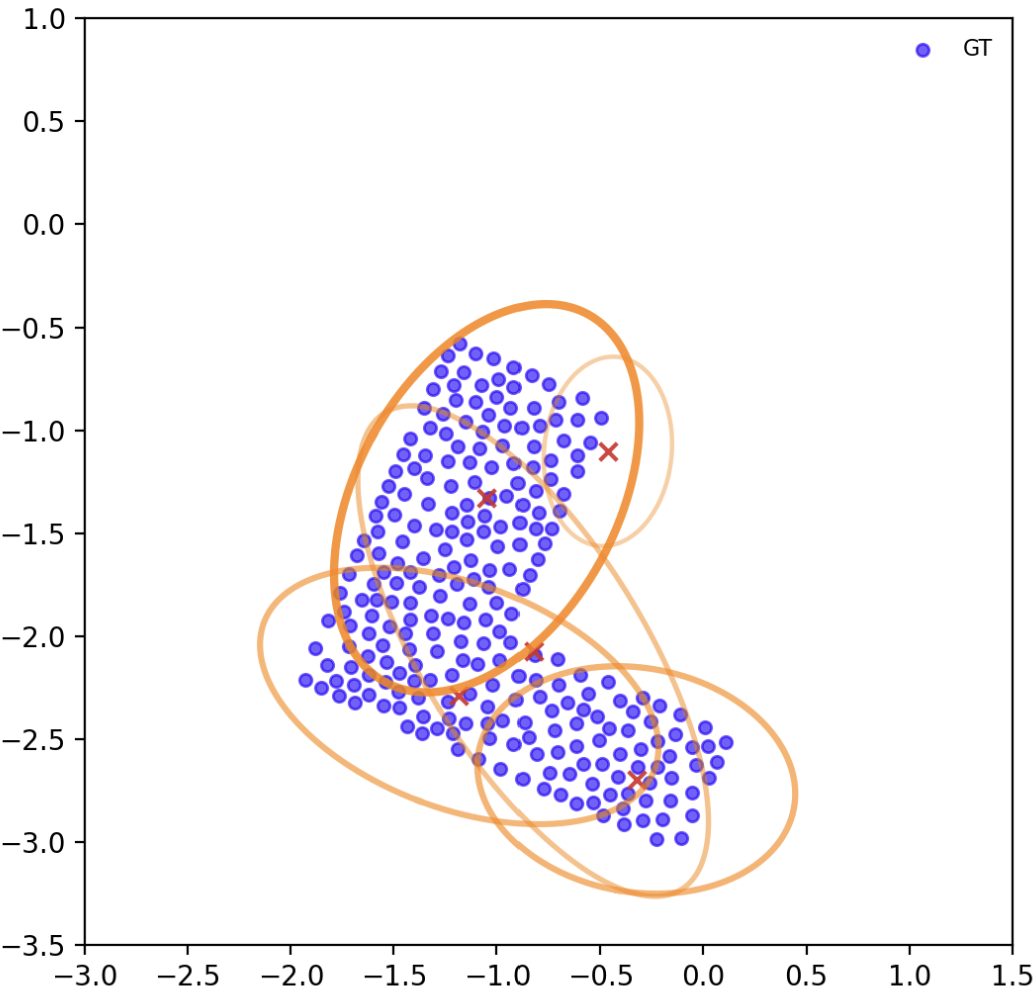}
        \caption{}
        \label{fig:ablation_full}
    \end{subfigure}
    \hspace{0.015\textwidth}
    \begin{subfigure}[b]{0.15\textwidth}
        \includegraphics[width=\textwidth]{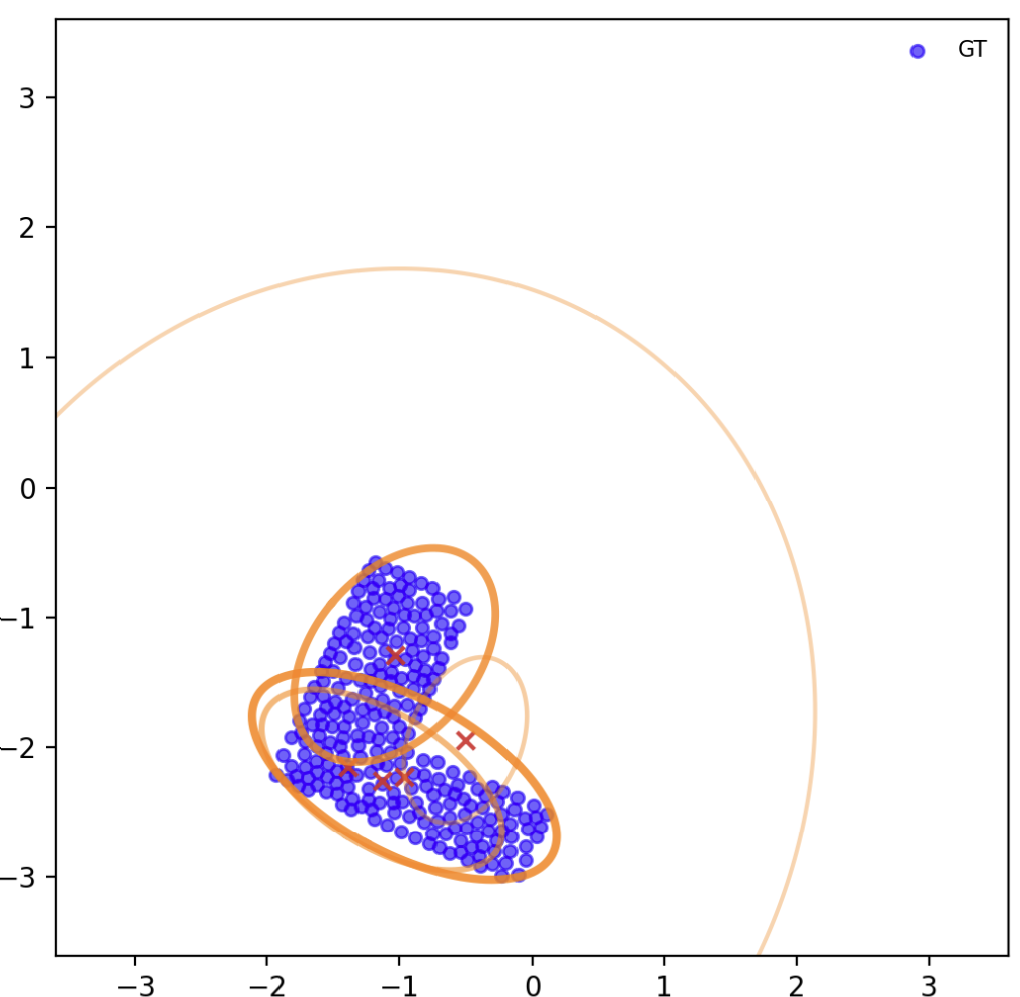}
        \caption{}
        \label{fig:ablation_no_empty}
    \end{subfigure}
    \hspace{0.015\textwidth}
    \begin{subfigure}[b]{0.15\textwidth}
        \includegraphics[width=\textwidth]{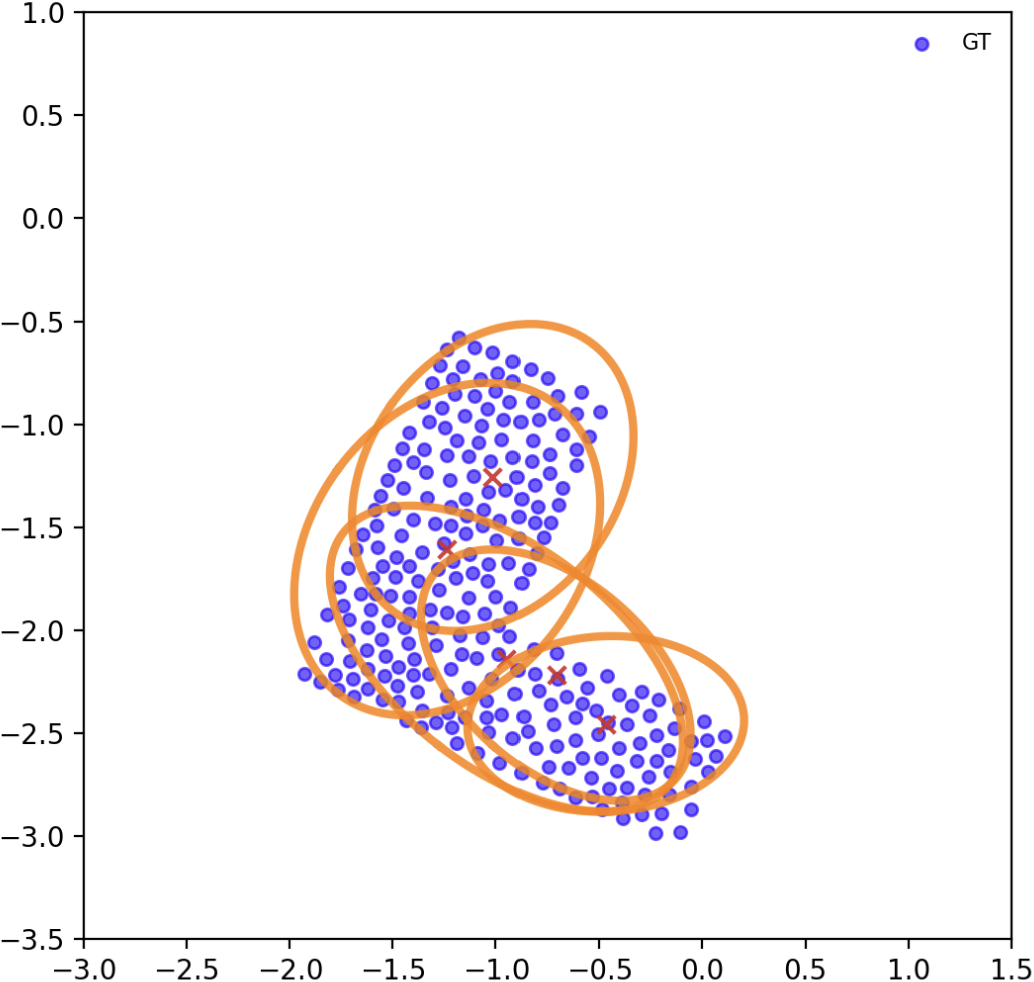}
        \caption{}
        \label{fig:ablation_no_nll}
    \end{subfigure}
    \caption{Ablation study on the loss function in the single-sofa case.
    (a) full loss function, (b) without the empty penalty term, and (c) without the NLL term.}
    \label{fig:ablation_loss}
\end{figure} \vspace{-1em}

\subsection{Comparison with Ellipsoid-based IPC}

We compare GM-IPC to single-ellipsoid IPC along two axes: \emph{(i)} compactness at matched coverage and \emph{(ii)} confidence at matched set size (equal union area of ellipsoids).

From Table~\ref{tab:validity}, we observe that GM-IPC achieves higher mean inclusion rates than the baseline across all four scenarios. Compactness gains are most pronounced as scene complexity increases. In the single-sofa case, compactness is slightly lower, likely due to the difficulty of training the network to align five distinct components compactly, compared to fitting a single near-spherical ellipsoid.

To further assess confidence, we compare GM-IPC and Ellip-IPC under the same union area. Given a fixed target size $S$ which is near the ground-truth area, we select a confidence level to define the ellipsoidal boundary, compute the resulting Gaussian union area, and adjust the confidence level so that the union size equals $S$. Intuitively, a higher solved confidence for the same allowed area indicates greater statistical certainty of covering ground-truth points. Across all scenarios (see Table~\ref{tab:confidence}), GM-IPC consistently attains higher solved confidence than Ellip-IPC, indicating more reliable coverage for a given set size.

\begin{table}[h]
\centering
\begin{tabular}{lcc}
\toprule
\textbf{Scenario} & \textbf{Ellip-IPC} & \textbf{GM-IPC} \\
\midrule
Mixed       & 0.817 & 0.982 \\
Two sofas   & 0.773 & 0.935 \\
Single sofa & 0.953 & 0.990 \\
Single chair& 0.894 & 0.929 \\
\bottomrule
\end{tabular}
\caption{Average solved confidence level under equal union area of ellipsoids (30 trials per scenario; aggregated across steps).}
\label{tab:confidence}
\end{table} \vspace{-1em}

\begin{figure}[t]
    \centering
    \begin{subfigure}{0.45\linewidth}
        \includegraphics[width=\linewidth]{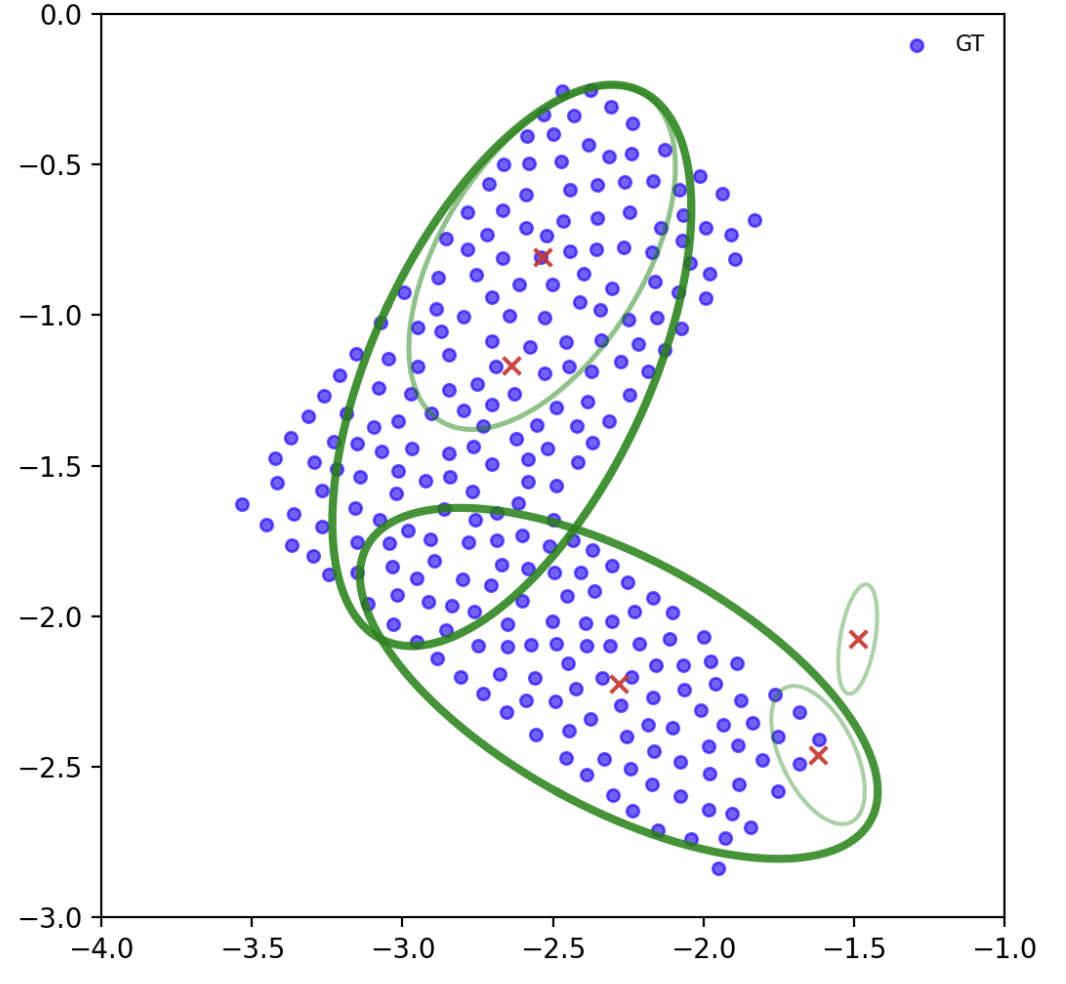}
        \caption{}
    \end{subfigure}\hfill
    \begin{subfigure}{0.45\linewidth}
        \includegraphics[width=\linewidth]{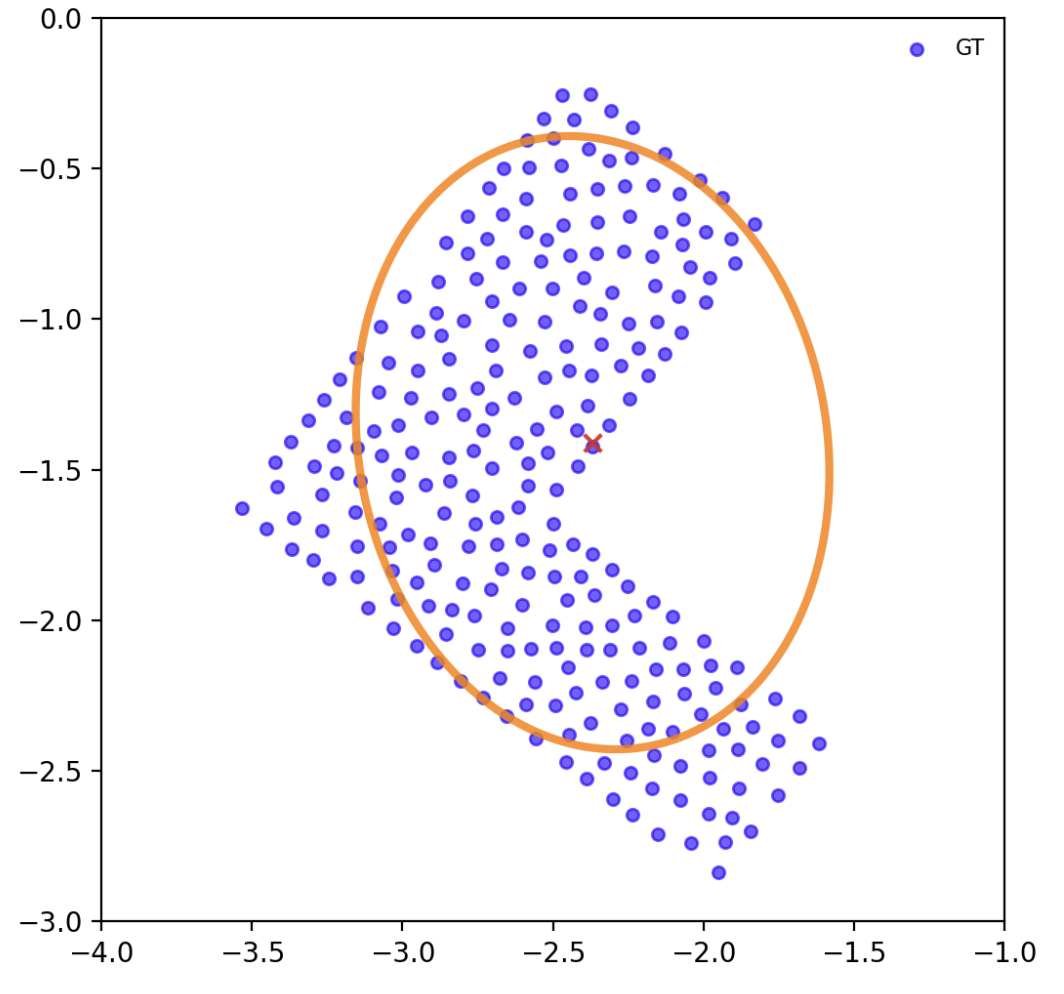}
        \caption{}
    \end{subfigure}

    \vspace{0.5em}

    \begin{subfigure}{0.45\linewidth}
        \includegraphics[width=\linewidth]{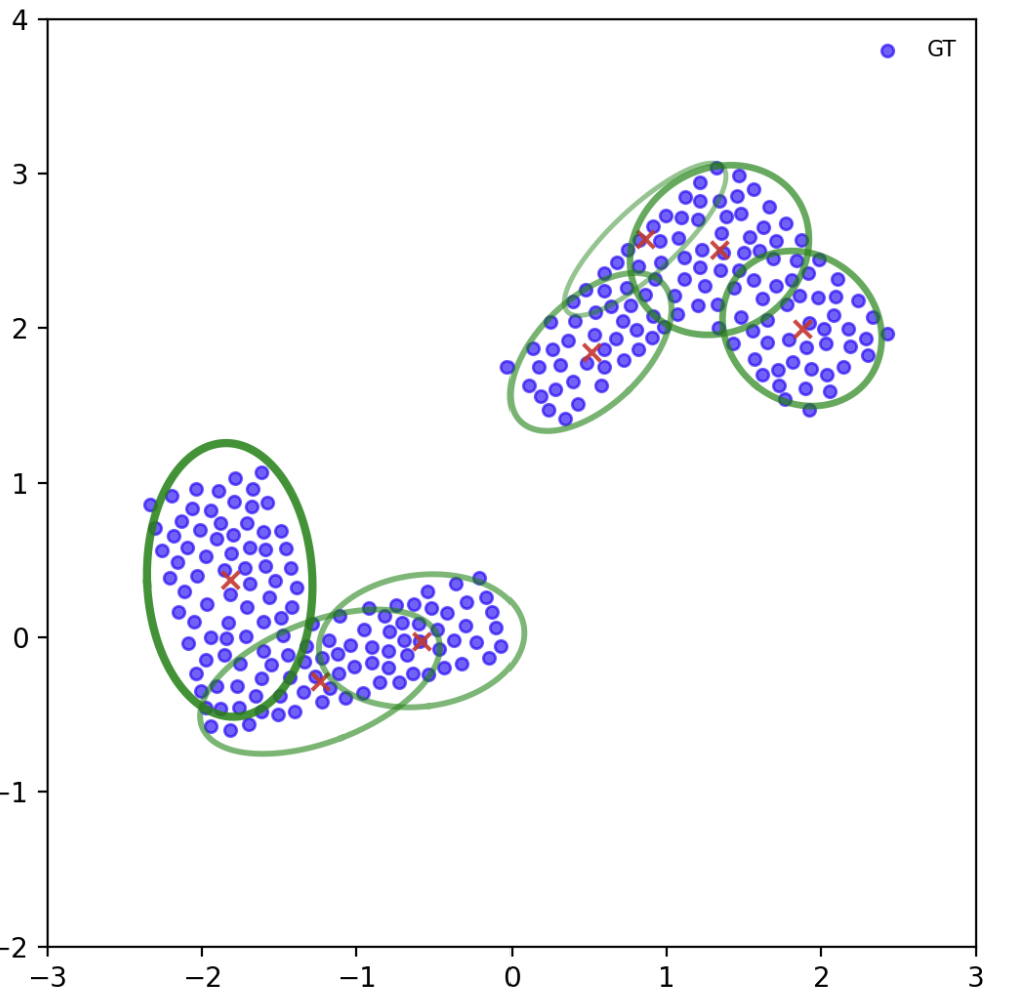}
        \caption{}
    \end{subfigure}\hfill
    \begin{subfigure}{0.45\linewidth}
        \includegraphics[width=\linewidth]{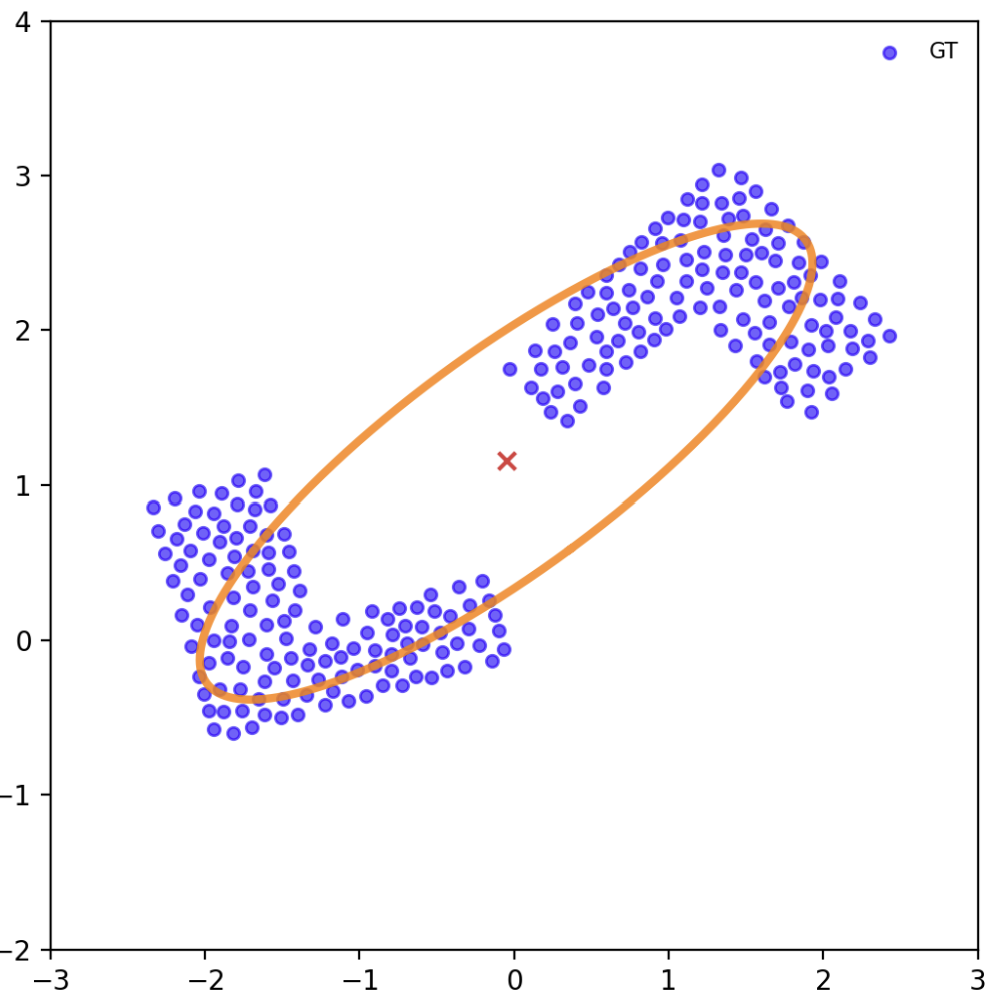}
        \caption{}
    \end{subfigure}
    \caption{Coverage under equal union area of ellipsoids:
    (a) GMM single-sofa, (b) Ellipsoid single-sofa, 
    (c) GMM multi-sofa, (d) Ellipsoid multi-sofa.}
    \label{fig:sofa_comparison}
\end{figure}

\subsection{Application in Navigation} \vspace{-0.5em}

To demonstrate the contribution of GM-IPC’s compact fitting ability to safe navigation, we conducted experiments in multi-object scenarios (two sofas, and mixed obstacles, i.e. one chair plus one sofa) with active online navigation. The overall navigation pipeline is described in Section~\ref{sec:nav_pip}.  

For each scenario, we ran 50 independent trials, randomizing the start and goal positions as well as the obstacle poses. The start and goal positions were chosen near the $\pm 5.0$ boundary of the environment to ensure that the robot traversed sufficiently long trajectories. This setting both increases the difficulty of obstacle-avoidance tasks and guarantees complete data collection.  

During each trial, we recorded the navigation success rate (defined as reaching the goal without collision within the 500 steps), the total number of steps, the trajectory length, and the path efficiency. Path efficiency is computed as the ratio between the Euclidean distance from start to goal and the actual trajectory length. Among these metrics, success rate reflects the safety guarantee, while total steps and path efficiency quantify the efficiency benefits of GM modeling (see Table~\ref{tab:nav_eval}).

\begin{table}[h!]
\centering
\caption{Navigation evaluation results across scenarios (50 trials each).}
\label{tab:nav_eval}
\begin{tabular}{llccccc}
\toprule
Scenario & Model & Steps & Len & Eff. & Ctrl & Succ. \\
\midrule
\multirow{2}{*}{Mixed} 
  & Ellip-1 & 434.2 & 12.69 & 0.90 & 0.18 & 0.76 \\
  & Ellip-2 & 408.3 & 11.95 & 0.95 & 0.12 & 0.70 \\
  & GMM   & \textbf{403.1} & \textbf{11.62} & \textbf{0.97} & \textbf{0.12} & \textbf{0.88} \\
\midrule
\multirow{2}{*}{Multi Sofa} 
  & Ellip-1 & 440.0 & 12.79 & 0.89 & 0.22 & 0.74 \\
  & Ellip-2 & 412.4 & 12.08 & 0.94 & 0.13 & 0.86 \\
  & GMM   & \textbf{397.2} & \textbf{11.42} & \textbf{0.99} & \textbf{0.13} & \textbf{0.90} \\
\bottomrule
\end{tabular}
\end{table}

From Table~\ref{tab:nav_eval}, we can tell that GM-IPC outperforms in both scenarios in all metrics. This shows that Gaussian mixture modeling of uncertainty sets provides not only more compact and efficient uncertainty sets, but also safer guidance for obstacle avoidance.

From Figure~\ref{fig:gmm_nav_traj}, 2 typical cases illustrate where GMM-IPC provides stronger navigation support. In the top row, when only a narrow gap exists between two obstacles, the ellipsoid-based IPC is overly conservative: although passing through the gap is the most effective path, the robot is forced to detour, whereas GM-IPC’s compact representation enables it to follow a more efficient trajectory.

In the bottom row, the limited expressive power of a single ellipsoid often fails to fully cover the ground-truth obstacle region where there are fewer perceived points. In contrast, GM-IPC employs multiple ellipsoids to capture different possible obstacle configurations, including those hidden in the camera’s blind spots. Thus, regions poorly covered by perceived points are still better represented than with ellipsoid-based IPC, demonstrating GM-IPC’s greater robustness under unstable perception.

\begin{figure}[t]
    \centering
    \begin{subfigure}{0.48\linewidth}
        \includegraphics[width=\linewidth]{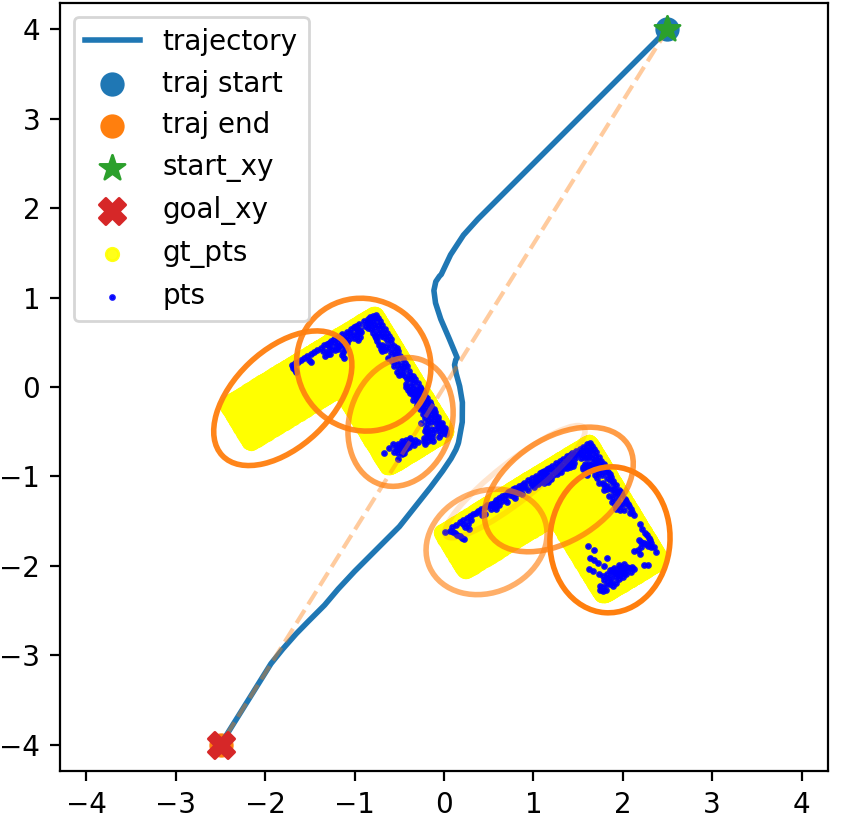}
        \caption{}
        \label{fig:gmm_nav_traj:a}
    \end{subfigure}
    \hfill
    \begin{subfigure}{0.48\linewidth}
        \includegraphics[width=\linewidth]{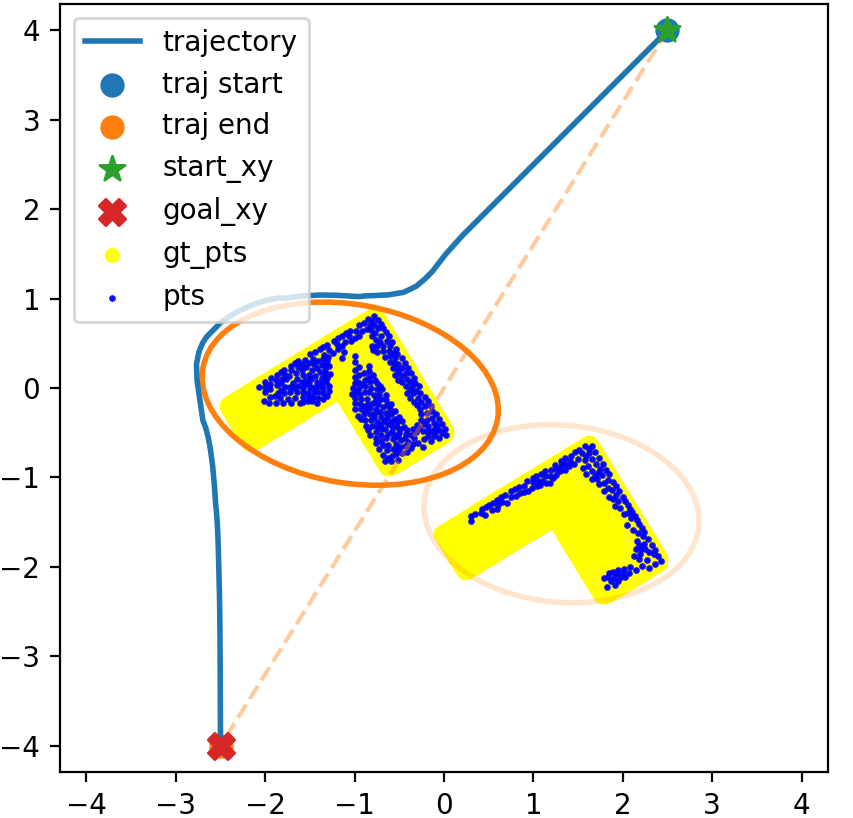}
        \caption{}
        \label{fig:gmm_nav_traj:b}
    \end{subfigure}

    \vspace{0.5em}

    \centering
    \begin{subfigure}{0.48\linewidth}
        \includegraphics[width=\linewidth]{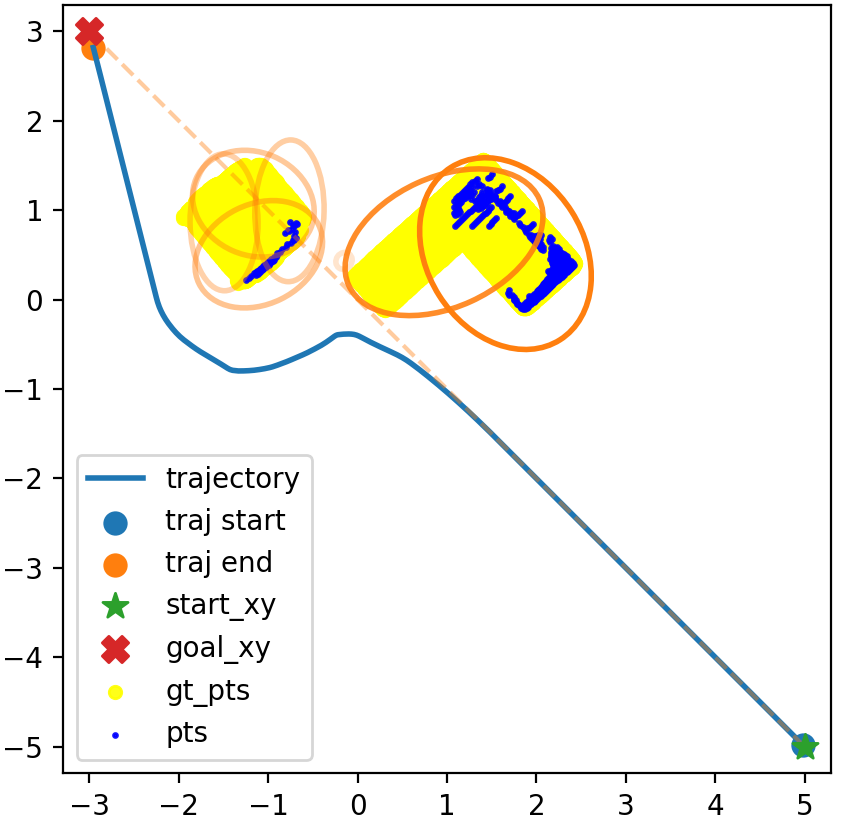}
        \caption{}
        \label{fig:gmm_nav_traj:c}
    \end{subfigure}
    \hfill
    \begin{subfigure}{0.48\linewidth}
        \includegraphics[width=\linewidth]{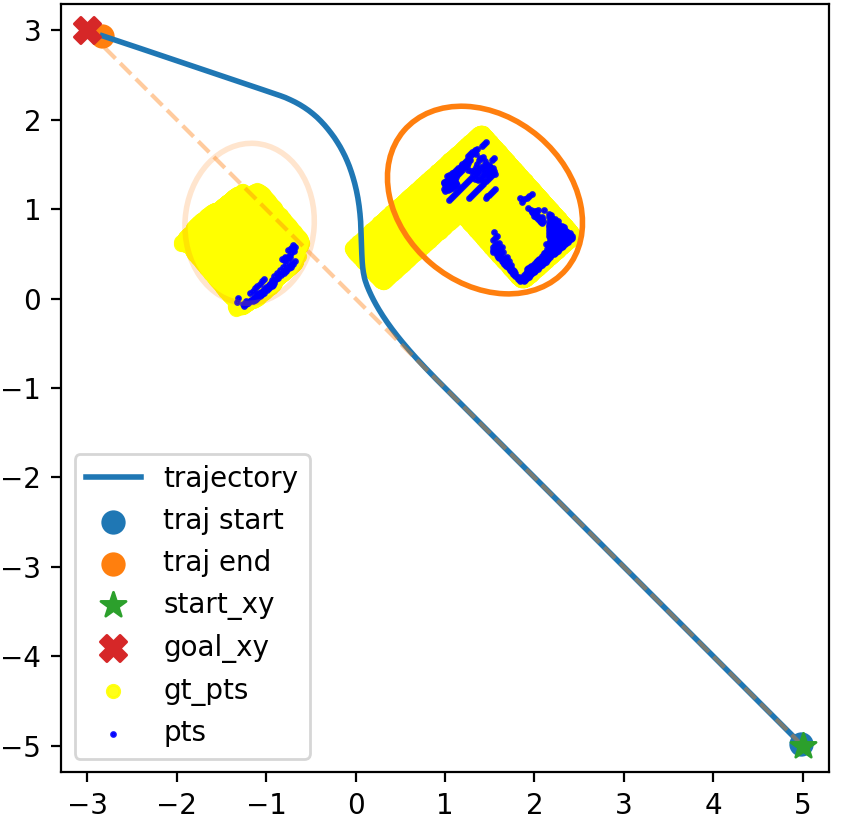}
        \caption{}
        \label{fig:gmm_nav_traj:d}
    \end{subfigure}

    \centering
    \begin{subfigure}{0.45\linewidth}
        \includegraphics[width=\linewidth]{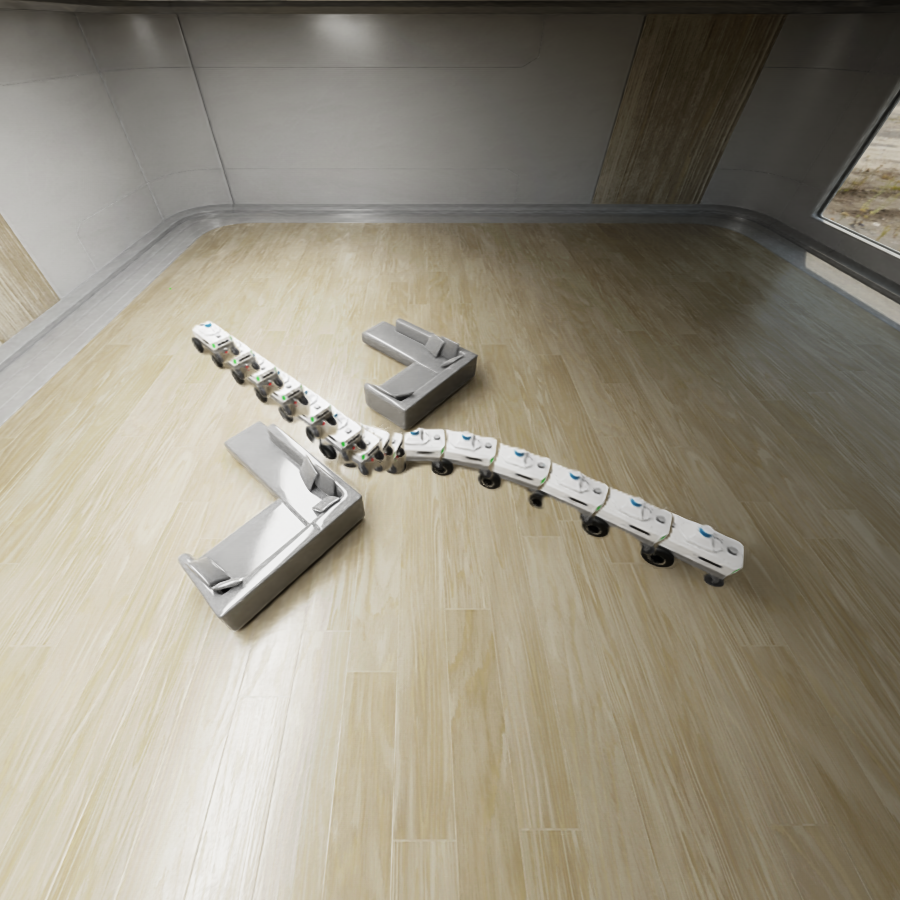}
        \caption{}
        \label{fig:gmm_nav_traj:e}
    \end{subfigure}
    \hfill
    \begin{subfigure}{0.45\linewidth}
        \includegraphics[width=\linewidth]{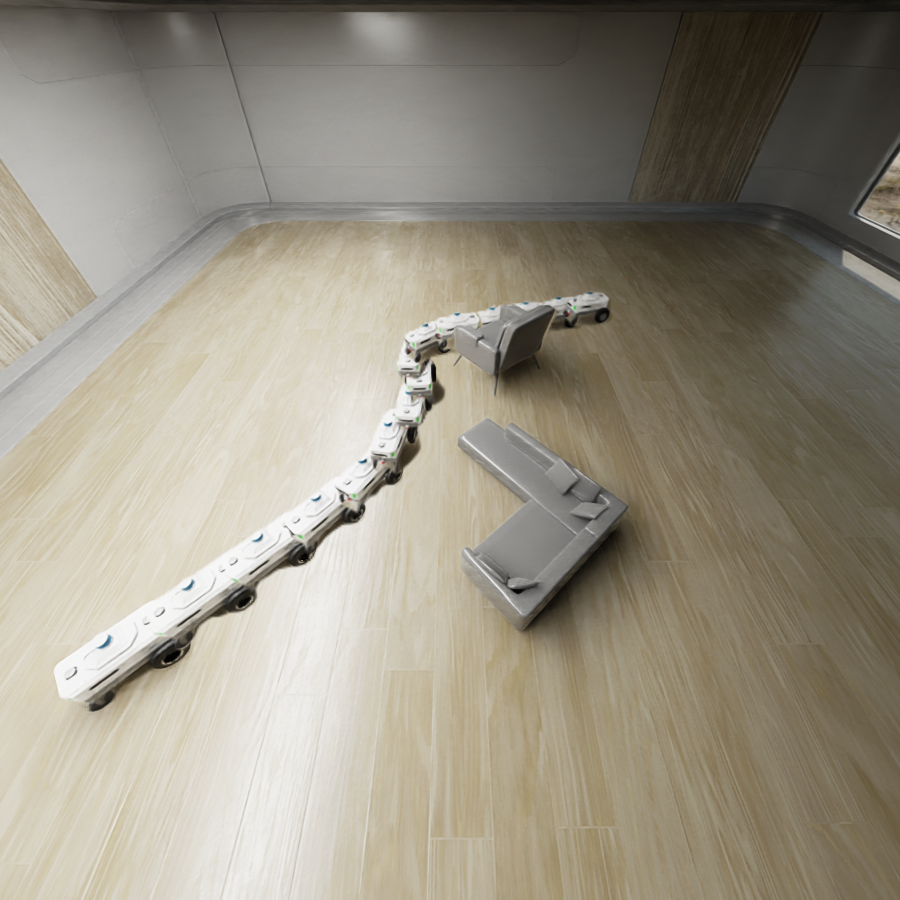}
        \caption{}
        \label{fig:gmm_nav_traj:f}
    \end{subfigure}
    
    \caption{Navigation trajectories in multi-object scenarios using IPC. Yellow regions indicate the ground-truth obstacles, and blue points are the observations from the perception module. (a, e) Two-sofa case with GM-IPC, (b) Two-sofa case with Ellip-IPC, (c, f) Mixed case with GM-IPC, (d) Mixed case with Ellip-IPC.}
    \label{fig:gmm_nav_traj}
\end{figure} \vspace{-0.5em}

\section{Conclusion} \vspace{-0.7em}

We presented \emph{GM-IPC}, a probabilistic extension of Inverse Perception Contracts that represents fine-grained perception uncertainty as a union of confidence ellipsoids induced by a Gaussian mixture. The formulation retains a formal coverage interpretation, and the learning objective couples an inclusion-driven term with regularizers for specialization and compactness. At the trial level, a PAC-style generalization bound links empirical inclusion loss to true miscoverage, and a consistency corollary shows that empirical risk minimization attains optimal population risk within the hypothesis class. We integrated GM-IPC with an MPC-CBF planner by converting mixture components into barrier functions and adapting constraint tightness to the predicted significance of obstacle regions. Experiments in simulated indoor navigation demonstrate that GM-IPC produces valid and compact uncertainty sets across scenes of varying complexity, often activating only the components needed for coverage. Relative to single-ellipsoid IPC, GM-IPC achieves comparable or higher inclusion and stepwise validity, higher solved confidence under equal union area, and improved compactness in multi-object and heterogeneous scenarios, enabling less conservative and more adaptive motion while maintaining probabilistic safety.

While GM-IPC models spatially structured, multimodal uncertainty, the reliance on a fixed maximum number of Gaussian components may limit scalability in more complex scenes. Future work will investigate adaptive component selection and more efficient uncertainty representations that balance expressiveness with real-time performance.
\vspace{-0.8em}


\bibliography{ref}

\begin{thebibliography}{10}
\providecommand{\url}[1]{#1}
\csname url@samestyle\endcsname
\providecommand{\newblock}{\relax}
\providecommand{\bibinfo}[2]{#2}
\providecommand{\BIBentrySTDinterwordspacing}{\spaceskip=0pt\relax}
\providecommand{\BIBentryALTinterwordstretchfactor}{4}
\providecommand{\BIBentryALTinterwordspacing}{\spaceskip=\fontdimen2\font plus
\BIBentryALTinterwordstretchfactor\fontdimen3\font minus \fontdimen4\font\relax}
\providecommand{\BIBforeignlanguage}[2]{{%
\expandafter\ifx\csname l@#1\endcsname\relax
\typeout{** WARNING: IEEEtran.bst: No hyphenation pattern has been}%
\typeout{** loaded for the language `#1'. Using the pattern for}%
\typeout{** the default language instead.}%
\else
\language=\csname l@#1\endcsname
\fi
#2}}
\providecommand{\BIBdecl}{\relax}
\BIBdecl

\bibitem{dean2020robust}
S.~Dean, N.~Matni, B.~Recht, and V.~Ye, ``Robust guarantees for perception-based control,'' in \emph{Learning for Dynamics and Control}.\hskip 1em plus 0.5em minus 0.4em\relax PMLR, 2020, pp. 350--360.

\bibitem{shao2024uncertainty}
W.~Shao, J.~Xu, Z.~Cao, H.~Wang, and J.~Li, ``Uncertainty-aware prediction and application in planning for autonomous driving: Definitions, methods, and comparison,'' \emph{arXiv preprint arXiv:2403.02297}, 2024.

\bibitem{berberich2025overview}
J.~Berberich and F.~Allg{\"o}wer, ``An overview of systems-theoretic guarantees in data-driven model predictive control,'' \emph{Annual Review of Control, Robotics, and Autonomous Systems}, vol.~8, no.~1, pp. 77--100, 2025.

\bibitem{guo2017calibration}
C.~Guo, G.~Pleiss, Y.~Sun, and K.~Q. Weinberger, ``On calibration of modern neural networks,'' in \emph{International conference on machine learning}.\hskip 1em plus 0.5em minus 0.4em\relax PMLR, 2017, pp. 1321--1330.

\bibitem{fontana2023conformal}
M.~Fontana, G.~Zeni, and S.~Vantini, ``Conformal prediction: a unified review of theory and new challenges,'' \emph{Bernoulli}, vol.~29, no.~1, pp. 1--23, 2023.

\bibitem{hsieh2022verifying}
C.~Hsieh, Y.~Li, D.~Sun, K.~Joshi, S.~Misailovic, and S.~Mitra, ``Verifying controllers with vision-based perception using safe approximate abstractions,'' \emph{IEEE Transactions on Computer-Aided Design of Integrated Circuits and Systems}, vol.~41, no.~11, pp. 4205--4216, 2022.

\bibitem{sun2024learning}
D.~Sun, B.~C. Yang, and S.~Mitra, ``Learning-based inverse perception contracts and applications,'' in \emph{IEEE International Conference on Robotics and Automation (ICRA)}.\hskip 1em plus 0.5em minus 0.4em\relax IEEE, 2024, pp. 11\,612--11\,618.

\bibitem{Caesar2020nuScenes}
H.~Caesar, V.~Bankiti, A.~H. Lang, and et~al., ``nuscenes: A multimodal dataset for autonomous driving,'' in \emph{IEEE/CVF Conference on Computer Vision and Pattern Recognition (CVPR)}, 2020, pp. 11\,621--11\,631.

\bibitem{Lang2019PointPillars}
A.~H. Lang, S.~Vora, H.~Caesar, L.~Zhou, J.~Yang, and O.~Beijbom, ``Pointpillars: Fast encoders for object detection from point clouds,'' in \emph{IEEE/CVF Conference on Computer Vision and Pattern Recognition (CVPR)}, 2019, pp. 12\,689--12\,697.

\bibitem{Philion2020LiftSplatShoot}
J.~Philion and S.~Fidler, ``Lift, splat, shoot: Encoding images from arbitrary camera rigs by implicitly unprojecting to 3d,'' in \emph{European Conference on Computer Vision (ECCV)}, 2020.

\bibitem{Janai2020AutonomousVisionSurvey}
J.~Janai, F.~Güney, A.~Behl, and A.~Geiger, ``Computer vision for autonomous vehicles: Problems, datasets and state-of-the-art,'' \emph{International Journal of Computer Vision}, vol. 128, no.~2, pp. 693--742, 2020.

\bibitem{wang2019pseudo}
Y.~Wang, W.-L. Chao, D.~Garg, B.~Hariharan, M.~Campbell, and K.~Q. Weinberger, ``Pseudo-lidar from visual depth estimation: Bridging the gap in 3d object detection for autonomous driving,'' in \emph{Proceedings of the IEEE/CVF conference on computer vision and pattern recognition}, 2019, pp. 8445--8453.

\bibitem{votenet}
C.~R. Qi, O.~Litany, K.~He, and L.~J. Guibas, ``Deep hough voting for 3d object detection in point clouds,'' in \emph{Proc. of the IEEE/CVF Conference on Computer Vision and Pattern Recognition (CVPR)}, 2019, pp. 9277--9286.

\bibitem{liu2021group}
Z.~Liu, Z.~Zhang, Y.~Cao, H.~Hu, and X.~Tong, ``Group-free 3d object detection via transformers,'' in \emph{Proceedings of the IEEE/CVF international conference on computer vision}, 2021, pp. 2949--2958.

\bibitem{yang2023sam3d}
Y.~Yang, X.~Wu, T.~He, H.~Zhao, and X.~Liu, ``Sam3d: Segment anything in 3d scenes,'' \emph{arXiv preprint arXiv:2306.03908}, 2023.

\bibitem{lindemann2021robust}
L.~Lindemann, M.~Cleaveland, Y.~Kantaros, and G.~J. Pappas, ``Robust motion planning in the presence of estimation uncertainty,'' in \emph{IEEE Conference on Decision and Control (CDC)}.\hskip 1em plus 0.5em minus 0.4em\relax IEEE, 2021, pp. 5205--5212.

\bibitem{resiliency}
A.~Khazraei, H.~Pfister, and M.~Pajic, ``Resiliency of perception-based controllers against attacks,'' in \emph{Learning for Dynamics and Control Conference}.\hskip 1em plus 0.5em minus 0.4em\relax PMLR, 2022, pp. 713--725.

\bibitem{astorga2023perception}
A.~Astorga, C.~Hsieh, P.~Madhusudan, and S.~Mitra, ``Perception contracts for safety of ml-enabled systems,'' \emph{Proceedings of the ACM on Programming Languages}, vol.~7, no. OOPSLA2, pp. 2196--2223, 2023.

\bibitem{chen2025hyperdimensional}
L.~Chen, J.~Wang, T.~Mortlock, P.~Khargonekar, and M.~A. Al~Faruque, ``Hyperdimensional uncertainty quantification for multimodal uncertainty fusion in autonomous vehicles perception,'' in \emph{Proceedings of the Computer Vision and Pattern Recognition Conference}, 2025, pp. 22\,306--22\,316.

\bibitem{hyper_seg}
T.~Sur, S.~Mukherjee, K.~Rahaman, S.~Chaudhuri, M.~H. Khan, and B.~Banerjee, ``Hyperbolic uncertainty-aware few-shot incremental point cloud segmentation,'' in \emph{Proceedings of the Computer Vision and Pattern Recognition Conference}, 2025, pp. 11\,810--11\,821.

\bibitem{ren2024recursively}
K.~Ren, C.~Chen, H.~Sung, H.~Ahn, I.~M. Mitchell, and M.~Kamgarpour, ``Recursively feasible chance-constrained model predictive control under gaussian mixture model uncertainty,'' \emph{IEEE Transactions on Control Systems Technology}, 2024.

\bibitem{em_algo}
A.~P. Dempster, N.~M. Laird, and D.~B. Rubin, ``Maximum likelihood from incomplete data via the em algorithm,'' \emph{Journal of the royal statistical society: series B (methodological)}, vol.~39, no.~1, pp. 1--22, 1977.

\bibitem{bishop2006pattern}
C.~M. Bishop and N.~M. Nasrabadi, \emph{Pattern recognition and machine learning}.\hskip 1em plus 0.5em minus 0.4em\relax Springer, 2006, vol.~4, no.~4.

\bibitem{bartlett2002rademacher}
P.~L. Bartlett and S.~Mendelson, ``Rademacher and gaussian complexities: Risk bounds and structural results,'' \emph{Journal of Machine Learning Research}, vol.~3, pp. 463--482, 2002.

\bibitem{shalev2014understanding}
S.~Shalev-Shwartz and S.~Ben-David, \emph{Understanding Machine Learning: From Theory to Algorithms}.\hskip 1em plus 0.5em minus 0.4em\relax Cambridge University Press, 2014.

\bibitem{ames2019cbf}
A.~D. Ames, S.~Coogan, M.~Egerstedt, G.~Notomista, K.~Sreenath, and P.~Tabuada, ``Control barrier functions: Theory and applications,'' in \emph{European Control Conference (ECC)}.\hskip 1em plus 0.5em minus 0.4em\relax Ieee, 2019, pp. 3420--3431.

\bibitem{mpccbf}
J.~Zeng, B.~Zhang, and K.~Sreenath, ``Safety-critical model predictive control with discrete-time control barrier function,'' in \emph{American Control Conference (ACC)}, 2021, pp. 3882--3889.

\bibitem{kroese2013handbook}
D.~P. Kroese, T.~Taimre, and Z.~I. Botev, \emph{Handbook of monte carlo methods}.\hskip 1em plus 0.5em minus 0.4em\relax John Wiley \& Sons, 2013.

\end{thebibliography}
\bibliographystyle{IEEEtran}
\end{document}